\documentclass[twoside]{article}

\usepackage[accepted]{aistats2023}

\RequirePackage{amsthm,amsmath,amsfonts,amssymb}
\RequirePackage[colorlinks,citecolor=blue,urlcolor=blue]{hyperref}
\RequirePackage{graphicx}
\RequirePackage{url, mathtools, subcaption, caption, tikz, tikz-3dplot, multicol, nicematrix, nicefrac, microtype, xcolor,thmtools, thm-restate, marginnote}
\RequirePackage[capitalize]{cleveref}
\usetikzlibrary{shapes,decorations,arrows,calc,arrows.meta,fit,positioning}

\theoremstyle{plain}

\newtheorem*{theorem*}{Theorem}
\newtheorem{theorem}{Theorem}[section]
\newtheorem{lemma}[theorem]{Lemma}
\newtheorem{proposition}[theorem]{Proposition}
\newtheorem{corollary}[theorem]{Corollary}
\theoremstyle{remark}
\newtheorem{assumption}{Assumption}
\newtheorem{definition}{Definition}

\makeatletter
\newenvironment{proof*}[1][\proofname]{\par
  \pushQED{\qed}%
  \normalfont \partopsep=\z@skip \topsep=\z@skip
  \trivlist
  \item[\hskip\labelsep
        \itshape
    #1\@addpunct{.}]\ignorespaces
}{%
  \popQED\endtrivlist\@endpefalse
}
\makeatother

\tikzset{
  -Latex,auto,node distance =1 cm and 1 cm,semithick,
  state/.style ={ellipse, draw, minimum width = 0.7 cm},
  point/.style = {circle, draw, inner sep=0.04cm,fill,node contents={}},
  bidirected/.style={Latex-Latex,dashed},
  el/.style = {inner sep=2pt, align=left, sloped}
}

\tdplotsetmaincoords{60}{120}

\NiceMatrixOptions{code-for-first-row = \color{blue}, 
code-for-last-col = \color{blue}}

\DeclareMathOperator{\Aut}{Aut}

\usepackage[round]{natbib}

\begin{document}

%
\runningtitle{Indeterminacy and Strong Identifiability in Generative Models}

%



\def\bfA{\mathbf{A}}
\def\bfB{\mathbf{B}}
\def\bfC{\mathbf{C}}
\def\bfD{\mathbf{D}}
\def\bfE{\mathbf{E}}
\def\bfF{\mathbf{F}}
\def\bfG{\mathbf{G}}
\def\bfH{\mathbf{H}}
\def\bfI{\mathbf{I}}
\def\bfJ{\mathbf{J}}
\def\bfK{\mathbf{K}}
\def\bfL{\mathbf{L}}
\def\bfM{\mathbf{M}}
\def\bfN{\mathbf{N}}
\def\bfO{\mathbf{O}}
\def\bfP{\mathbf{P}}
\def\bfQ{\mathbf{Q}}
\def\bfR{\mathbf{R}}
\def\bfS{\mathbf{S}}
\def\bfT{\mathbf{T}}
\def\bfU{\mathbf{U}}
\def\bfV{\mathbf{V}}
\def\bfW{\mathbf{W}}
\def\bfX{\mathbf{X}}
\def\bfY{\mathbf{Y}}
\def\bfZ{\mathbf{Z}}

\def\bbA{\mathbb{A}}
\def\bbB{\mathbb{B}}
\def\bbC{\mathbb{C}}
\def\bbD{\mathbb{D}}
\def\bbE{\mathbb{E}}
\def\bbF{\mathbb{F}}
\def\bbG{\mathbb{G}}
\def\bbH{\mathbb{H}}
\def\bbI{\mathbb{I}}
\def\bbJ{\mathbb{J}}
\def\bbK{\mathbb{K}}
\def\bbL{\mathbb{L}}
\def\bbM{\mathbb{M}}
\def\bbN{\mathbb{N}}
\def\bbO{\mathbb{O}}
\def\bbP{\mathbb{P}}
\def\bbQ{\mathbb{Q}}
\def\bbR{\mathbb{R}}
\def\bbS{\mathbb{S}}
\def\bbT{\mathbb{T}}
\def\bbU{\mathbb{U}}
\def\bbV{\mathbb{V}}
\def\bbW{\mathbb{W}}
\def\bbX{\mathbb{X}}
\def\bbY{\mathbb{Y}}
\def\bbZ{\mathbb{Z}}

\def\calA{\mathcal{A}}
\def\calB{\mathcal{B}}
\def\calC{\mathcal{C}}
\def\calD{\mathcal{D}}
\def\calE{\mathcal{E}}
\def\calF{\mathcal{F}}
\def\calG{\mathcal{G}}
\def\calH{\mathcal{H}}
\def\calI{\mathcal{I}}
\def\calJ{\mathcal{J}}
\def\calK{\mathcal{K}}
\def\calL{\mathcal{L}}
\def\calM{\mathcal{M}}
\def\calN{\mathcal{N}}
\def\calO{\mathcal{O}}
\def\calP{\mathcal{P}}
\def\calQ{\mathcal{Q}}
\def\calR{\mathcal{R}}
\def\calS{\mathcal{S}}
\def\calT{\mathcal{T}}
\def\calU{\mathcal{U}}
\def\calV{\mathcal{V}}
\def\calW{\mathcal{W}}
\def\calX{\mathcal{X}}
\def\calY{\mathcal{Y}}
\def\calZ{\mathcal{Z}}

\def\bfz{\mathbf{z}}
\def\bfx{\mathbf{x}}

\def\iid{i.i.d.\ } 
\def\ie{i.e.\ }
\def\eg{e.g.\ }
\newcommand{\kword}[1]{\smallskip\noindent\textbf{#1}\ \ }


\def\P{\bbP} 
\def\E{\bbE} 
\DeclarePairedDelimiterX\bigCond[2]{[}{]}{#1 \;\delimsize\vert\; #2}
\newcommand{\conditional}[3][]{\bbE_{#1}\bigCond*{#2}{#3}}
\def\Law{\mathcal{L}} 
\def\indicator{\mathds{1}} 

\def\borel{\calB} 
\def\sigAlg{\calA} 
\def\filtration{\calF} 
\def\grp{\calG} 

\makeatletter
\newcommand{\oset}[3][0.25ex]{%
  \mathrel{\mathop{#3}\limits^{
    \vbox to#1{\kern-2\ex@
    \hbox{$\scriptstyle#2$}\vss}}}}
\makeatother

\def\condind{{\perp\!\!\!\perp}} 
\def\equdist{\oset{\text{\rm\tiny d}}{=}} 
\def\equas{\oset{\text{\rm\tiny a.s.}}{=}} 
\def\equae{\oset{\text{\rm\tiny a.e.}}{=}} 
\def\simiid{\sim_{\mbox{\tiny iid}}} 

\def\onevec{\mathbf{1}}
\def\iden{\mathbf{I}} 
\def\supp{\text{\rm supp}}

\DeclarePairedDelimiter{\ceilpair}{\lceil}{\rceil}
\DeclarePairedDelimiter{\floor}{\lfloor}{\rfloor}
\newcommand{\argdot}{{\,\vcenter{\hbox{\tiny$\bullet$}}\,}} 

\def\UnifDist{\text{\rm Unif}}
\def\BetaDist{\text{\rm Beta}}
\def\ExpDist{\text{\rm Exp}}
\def\GammaDist{\text{\rm Gamma}}
\def\NormDist{\mathcal{N}}

\def\BernDist{\text{\rm Bernoulli}}
\def\BinomDist{\text{\rm Binomial}}
\def\PoissonPlus{\text{\rm Poisson}_{+}}
\def\PoissonDist{\text{\rm Poisson}}
\def\NBPlus{\text{\rm NB}_{+}}
\def\NBDist{\text{\rm NB}}
\def\GeomDist{\text{\rm Geom}}


\def\obs{\bfX}
\def\latent{\bfZ}
\def\genClass{\calF}
\def\distClass{\calP}
\def\obsModel{\calM}
\def\indet{\calA}
\def\id{\textrm{id}}

\newcommand{\sigalg}[1]{\mathcal{#1}}
\def\borel{\calB}
\def\refMeas{\mu}
\def\lebesgue{\lambda}


\definecolor{BBRCommentColor}{rgb}{0,0,.50}
\definecolor{JXCommentColor}{rgb}{0,0.50,.50}
\newcounter{margincounter}
\newcommand{\displaycounter}{{\arabic{margincounter}}}
\newcommand{\incdisplaycounter}{{\stepcounter{margincounter}\arabic{margincounter}}}
\newcommand{\BBRCOMMENT}[1]{\textcolor{BBRCommentColor}{$\,^{(\incdisplaycounter)}$}\marginpar{\scriptsize\textcolor{BBRCommentColor}{ {\tiny $(\displaycounter)$} #1}}}
\newcommand{\JXCOMMENT}[1]{\textcolor{JXCommentColor}{$\,^{(\incdisplaycounter)}$}\marginpar{\scriptsize\textcolor{JXCommentColor}{ {\tiny $(\displaycounter)$} #1}}}

\newcommand{\equivA}{\sim_{\indet(\obsModel)}}
\newcommand{\equivt}{\sim_{t}}
\newcommand{\equivtx}{\sim_{t(\bfx_m,\argdot)}}

\newcommand{\genTrans}{\vec{A}}

\twocolumn[

\aistatstitle{Indeterminacy in Generative Models:  \\ Characterization and Strong Identifiability}

\aistatsauthor{ Quanhan Xi \And Benjamin Bloem-Reddy }

\aistatsaddress{  \texttt{johnny.xi@stat.ubc.ca} \\ University of British Columbia   \And  \texttt{benbr@stat.ubc.ca} \\ University of British Columbia  
} 
]

\begin{abstract}
 Most modern probabilistic generative models, such as the variational autoencoder (VAE), have certain indeterminacies that are unresolvable even with an infinite amount of data. Different tasks tolerate different indeterminacies, however recent applications have indicated the need for \textit{strongly} identifiable models, in which an observation corresponds to a unique latent code. Progress has been made towards reducing model indeterminacies while maintaining flexibility, and recent work excludes many---but not all---indeterminacies. In this work, we motivate model-identifiability in terms of task-identifiability, then construct a theoretical framework for analyzing the indeterminacies of latent variable models, which enables their precise characterization in terms of the generator function and prior distribution spaces. We reveal that strong identifiability is possible even with highly flexible nonlinear generators, and give two such examples. One is a straightforward modification of iVAE \citep{Khem2020a}; the other uses triangular monotonic maps, leading to novel connections between optimal transport and identifiability.
\end{abstract}

\section{\uppercase{Introduction}} \label{sec:intro}

In generative models, indeterminacy refers to the situation where the latent values underlying observations cannot be uniquely inferred from any amount of empirical evidence. It is a structural issue occurring in generative models such as the variational auto-encoder (VAE) \citep{King2014} or independent component analysis (ICA) \citep{Comon1994, Hyv1999}. For example, it is well known that even linear Gaussian models are plagued by a rotational indeterminacy, where arbitrary rotations of the latent space result in equivalent observation distributions. Modern nonlinear (deep) generative models inherit such indeterminacies and many more.

Characterizing and reducing the indeterminacies of generative models has been referred to as \emph{identifiability}, and has been motivated as a way to address a variety of problems in representation learning such as disentangelment \citep{Loc2019, Halv2021, Yang2021, Klindt2021}, posterior collapse \citep{Wang2021b}, and causal representation learning \citep{Wang2021a, Scho2021, Lu2022}. Compared to linear models, which restrict indeterminacies to linear transformations of the latent space, nonlinear identifiability is more challenging, and has been the subject of much recent work. 

Consider a prototypical generative model,
\begin{align*}
Z_i \sim P_z \;, \quad \epsilon_i \sim \calN(0,\sigma^2) \;, \quad
X_i = f(Z_i) + \epsilon_i \;,
\end{align*}
where $P_z \in \distClass_z$ is a distribution on latent variables, $f \in \genClass$ is an injective generator function, and $X_i$ is an observation. 
Indeterminacies arise when more than one $(f,P_z)$ pair give rise to the same marginal distribution on observations, due to transformations of the latent space. 
Despite the seemingly disparate approaches taken in recent work to reduce indeterminacy, an intuitive trade-off appears: as more structure is added to the model through $\genClass$ and $\distClass_z\textbf{}$, indeterminacy is reduced and hence stronger identifiability results are obtained. 
Our first contribution is a framework for analyzing indeterminacy that abstracts away model specifics and formalises the trade-off as a general phenomenon, described by separating indeterminacy transformations into two possible sources. The first source is the set of transports that turn one distribution in $\distClass_z$ into another; denote this set by $\indet(\distClass_z)$. The second source, $\indet(\genClass)$, is the set of automorphisms of the latent space formed from two functions in $\genClass$ as $f_b^{-1} \circ f_a$
(See \cref{sec:indet:maps} for details). The main result of the analysis is that model indeterminacies must belong to \emph{both} sets.

\begin{theorem*}[Informal Statement of \Cref{thm:id_general}]
The set of indeterminacy transformations of a generative model $(\genClass,\distClass_z)$ is precisely $\indet(\genClass) \cap \indet(\distClass_z)$.
\end{theorem*}
This reduces the identifiability analysis for any suitable generative model to a well-defined mathematical problem. It also categorizes various methods to reduce indeterminacies in the model class---either constrain $\genClass$, as in linear models; constrain $\distClass_z$, e.g., non-Gaussians; or constrain both. Any of these approaches will reduce the size of the intersection. 
This describes many recent methods for nonlinear identifiability. Generator constraints include restrictions to certain optimal transport maps \citep{Wang2021b}, sparsity \citep{Moran2021, zheng2022} and restrictions on the Jacobian \citep{Gresele2021,buchholz2022}. Many methods also assume a diffeomorphic or analytic generator, which further reduces $\indet(\genClass)$. Direct constraints on $\distClass_z$ include non-Gaussianity \citep{stuhmer2020} and mixture distributions \citep{Kivva2022}. More commonly, latent distribution constraints are formulated in terms of dependence on an auxiliary variable \citep{Hyv2016, Hyv2017, Hyv2018, Khem2020a, Khem2020b, Halv2021, Klindt2021}, multiple views \citep{Loc2020}, when latent variables are perturbed via interventions \citep{Brehmer2022} or sparse mechanism shifts \citep{lachapelle2022, ahuja2022}. As an application of our framework, we show explicitly in \cref{sec:ME} how multiple environments \citep[e.g.,][]{Khem2020a} and multiple views \citep[e.g.,][]{Loc2020} can yield stronger identifiability, simply by reducing the set of indeterminacy transformations. 

Despite the progress in the references above, each of the identifiability results contained therein are weak, in the sense that non-trivial indeterminacies remain. An unanswered question is what constraints are required for strong identifiability, that is, pointwise uniqueness of the latent representation. It has long been seen as unattainable without major sacrifices to model flexibility. For example, permutation and scaling indeterminacies are considered fundamental in ICA models \citep{Comon1994}. Our second main contribution is to use our framework in a few different ways to specify strongly identifiable nonlinear models, without restrictive sacrifices in flexibility. The simplest is to freeze the latent distribution before training the generator, as is typically done in a VAE. Specifically, we show that freezing the priors in iVAE \citep{Khem2020a}, either from the outset or after some initial training, yields strong identifiability---with no further constraints on the generator class---when data from distinct environments (auxiliary information) are used (\cref{sec:ME}). We also show that monotonic triangular flow generators \citep{huang2018, Jaini2019,Wehenkel2019, Irons2021}, which are universal transports between fully supported distributions, are strongly identifiable even with a single environment, and with any latent distribution (\cref{sec:transports}).

Even when strong identifiability is unachievable, weakly identifiable models may be useful for tasks that tolerate the remaining indeterminacies. Our third main contribution is to formally relate model identifiability to task identifiability in \cref{sec:task:id}, providing precise conditions for when weak identifiability is good enough for identifiability of a particular downstream task, including some examples from the recent literature. One obvious conclusion is that strongly identifiable models are acceptable for any task based on the latent variables. 

Before proceeding, we note that any notion of model identifiability is an asymptotic property not achievable with finite data. However, model identifiability is an important quality for statistical inference, and in particular is necessary for the typical consistency guarantees \citep{Vander1998}. Though future work is required to assess the finite-sample properties of identifiable models, empirical evidence shows that even weak identifiability can recover ground truths in simulation studies \citep{Khem2020a, Sorrenson2020, Lu2022}.

\kword{Outline} 
In the rest of this section, we motivate our framework with two classical linear examples, and discuss identifiability at a high level from a downstream task-specific point of view, outlining when and what degree of identifiability is required. In \cref{sec:defs}, we define our mathematical framework and present our main technical results on model identifiability. After revisiting the task-specific view in detail in \cref{sec:task:id}, we then apply the framework to analyze the auxiliary information setting in \cref{sec:ME}, showing in detail how the iVAE fits into the framework, and how it can be easily adapted for strong identifiability. In \cref{sec:transports} we describe properties of triangular flows that yield strongly identifiable models, which can then be generalized to flows based on certain optimal transports.

\subsection{Factor Analysis and Linear ICA} \label{sec:fa:ica}

Factor analysis \citep{Lawley1962} is a linear generative model, 
\begin{align*}
    Z_i \sim \NormDist(0, I_{d_z }) , \; \quad
    \epsilon_i \sim \NormDist(\mu, I_{d_x }) , \; \quad
    X_i = F\, Z_i + \epsilon_i \;,
\end{align*}
where $\epsilon_i\condind Z_i$, and $F$ is a full-rank $d_x \times d_z$ matrix of so-called factor loadings. Here, $F$ is the only learnable parameter and $\distClass_z = \{ \NormDist(0, I_{d_z })\}$, a singleton.  
Linear ICA \citep{Comon1994} is structurally identical, but relaxes the assumption that $Z$ has a Gaussian distribution. Instead, it is parametrized by $(F, P_z)$, where $F$ is again full rank, and $P_z$ is required to have independent components and is learned along with $F$ from data. In other words, $\distClass_z$ is some collection of fully supported distributions with independent components, typically excluding Gaussians.

It is well known that factor analysis suffers from a rotational indeterminacy due to the Gaussian $P_z$. On the other hand, \cite{hastie2009} note that ICA is identical in form to factor analysis, but avoids the rotational indeterminacy via its non-Gaussian assumption. However, if scaling indeterminacies are fundamental to ICA \citep{hyv2001book}, why are they ignored in factor analysis?

Factor analysis does not suffer from scaling indeterminacy due to a fundamental difference in how indeterminacies arise in these two models. The factor model  fixes its Gaussian $P_z$, and the rotational indeterminacies in factor analysis can be characterized precisely as linear measure-preserving automorphisms of the standard Gaussian. Since scaling does not preserve the standard Gaussian, it is not an indeterminacy. On the other hand, ICA does not fix $P_z$. Due to this, any linear transformation of $P_z \in \calP_z$ to another $P_z' \in \calP_z$ is an indeterminacy, and a result of \citet[Thm.~10]{Comon1994} is that only scalings and permutations are possible when $\calP_z$ contains all independent distributions excluding Gaussians. 

Our \cref{thm:id_general} below generalizes these special cases. It can be used to show that  indeterminacies are completely characterized as measure-preserving automorphisms for fixed latent distributions (as in factor analysis), and as measure-transporting isomorphisms within $\calP_z$ otherwise (as in ICA). One takeaway is that strong identifiability appears to be much easier to obtain in the factor analysis setting. We investigate this further in \cref{sec:ME,sec:transports}.

\subsection{Why and How Much Identifiability}

In the recent literature on generative model identifiability, the question of why we care about identifiability typically appeals to recovering ground truth latent factors \citep{Khem2020a,Ahuja2021,Yang2021,Lu2022}. Besides requiring a philosophical position that asserts the objective reality of the latent variables \emph{and} that the model contains the ``true'' data generating distribution, the use of an  unidentifiable model---unable to be uniquely resolved from \emph{any} amount of empirical evidence---makes the recovery of ``true'' latent factors impossible without further, untestable assumptions about the model. Another perspective is to consider when model indeterminacies preserve certain observable quantities, for example distances between observations \citep{arva2018}. More generally, we might judge the importance of model identifiability in terms of whether or not a model can be used for its intended purpose. This opens the possibility that in some cases, weak identifiability may be sufficient; in others, it may not. To our knowledge, a characterization of the distinction has not been formalized in this setting.

To address these questions, we take a pragmatic approach based on downstream tasks that use the inferred latent variables. Informally, a task $t$ is a function of the model parameters $(f,P_z)$, of data, $\bfx_m = (x_1,\dotsc,x_m)$, and of a finite collection of points in the latent space, $\bfz_n = (z_1,\dotsc,z_n)$. There must be some procedure for selecting the latent points $\bfz_n$ represented by a selection function $\bfz_n = s(f,P_z, \bfx_m)$, e.g., $s(f, P_z, \bfx_m) \approx f^{-1}(\bfx_m)$. The entire task is given by $t(f,P_z, \bfx_m, s(f, P_z, \bfx_m))$. 

Some examples are generating synthetic or counterfactual data by shifting inferred latent variables \citep{higgins2017}, or causal discovery through independence tests \citep{monti2020,Khem2020a,Lu2022}. Both of these examples are studied in detail in \cref{sec:task:id}. For a model to be useful for a task, the resulting task output should be the same for all generative model parameters $(f,P_z)$ that yield the same marginal distribution for the observations (this is made precise in \cref{sec:task:id}). For example, \cref{figure::intro} illustrates a task that is constant on each set of points induced by the indeterminacy transformations (rotations). That guarantees that two fits of the same model to the same (infinite) data yield the same task values; weak model identifiability can be sufficient when the relevant task is insensitive to all model indeterminacies. Strongly identifiable models make all tasks trivially insensitive and therefore identifiable. 

\begin{figure}[tb]
    \centering
	\includegraphics[width = 0.75\columnwidth]{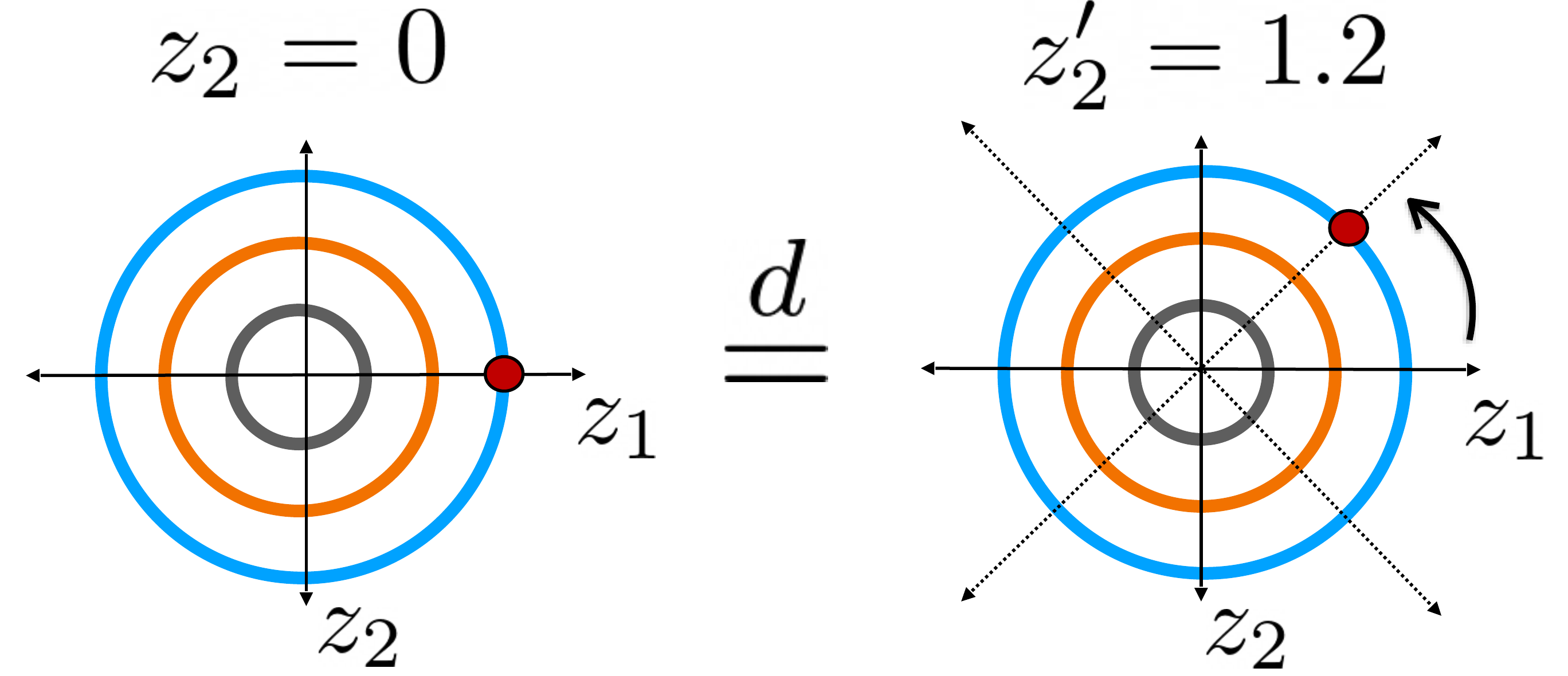}
	\caption{Rotational indeterminacy is fundamental in Gaussian i.i.d.\ models, where each point on the above orbits represent an equivalent solution. If a task gives constant output on each orbit, the task is identifiable even in the presence of model indeterminacies.}
	\label{figure::intro}
\end{figure}

\section{\uppercase{Identifiability in generative models}} \label{sec:defs}

Conceptually, our results are relatively simple, but some care is required to establish them rigorously. See \cref{sec:appx:definitions} and  \cref{sec:appx:proofs} for technical background and proofs. 

A Borel space $(\obs,\borel(\obs))$ is a topological space $\obs$ equipped with the $\sigma$-algebra generated by its open sets, denoted $\borel(\obs)$. All measurable spaces in this paper will be Borel spaces, and for convenience we leave the $\sigma$-algebra implicit in the notation, referring to the Borel space $(\obs,\borel(\obs))$ simply as $\obs$. Let $\latent$ and $\obs$ be two Borel spaces and $f \colon \latent \to \obs$ a measurable function. If $f$ is bijective with measurable inverse then it is called a Borel isomorphism of $\latent$ and $\obs$. If $\obs = \latent$ then $f$ is a Borel automorphism. We denote the set of Borel automorphisms of $\latent$ as $\Aut(\latent)$. 

We will use the notation $\obs$ and $\latent$ to represent the observable space and latent space, respectively. In practice, typically $\obs = \bbR^{d_x}$ and $\latent = \bbR^{d_z}$, $d_z \leq d_x$. For clarity, we will usually refer to these specific spaces, where \emph{almost everywhere} statements and probability densities are with respect to the Lebesgue measure; our results in \cref{sec:defs} and \cref{sec:task:id} apply to generic Borel measure spaces. We write $f_a \equae f_b$ for measurable functions as shorthand for equality almost everywhere on their domain.
The main objects of study in our theory are certain elements of $\Aut(\latent)$, as follows. 

Let $\mu$ and $\nu$ be two measures on $\latent$, and $A \in \Aut(\latent)$. $A$ is called a $(\mu,\nu)$\textbf{-measure isomorphism} whenever $A_{\#}\mu = \nu$ (equiv.\  $A^{-1}_{\#}\nu = \mu)$, and a $\mu$\textbf{-measure preserving automorphism} if $\mu = A_{\#}\mu = A^{-1}_{\#}\mu$.\footnote{$A_{\#}\mu$ ($A$ \emph{pushforward} $\mu$) can also be denoted $A_{*}\mu$, or $\mu\circ A^{-1}$ (the image measure), and is defined by $A_{\#}\mu(B) = \mu(A^{-1}(B))$, $B \in \borel(\latent)$.} 

\subsection{Model Definition} 
We model observations denoted $x_i$ as i.i.d.\ realizations of a random variable $X$ with distribution $P_x$ on $\obs$. Define a latent random variable $Z$ with corresponding realizations $z_i$, with distribution $P_z$ on the latent space $\latent$. We further assume the presence of some noise realization $\epsilon_i$ with some fixed distribution $P_{\epsilon}$ applied by some noise mechanism $g$. Let $f \colon \latent \to \obs$ be a measurable function, which we call the \emph{generator}. We define a \textit{generative model} as,
\begin{align}
Z_i \sim P_z \;, \quad \epsilon_i \sim P_\epsilon \;, \quad
X_i = g(f(Z_i), \epsilon_i) \;,
\label{eq:model}
\end{align}
with $Z_i \condind \epsilon_i$. 
For example, in a VAE or factor analysis, $\obs$ and $\latent$ are Euclidean spaces, and $g(f(z), \epsilon) = f(z) + \epsilon$ is additive noise. We will not be concerned with inferring the noise. Instead, we work with the assumption that the noise distribution and mechanism are fixed and have null effect on the probabilistic properties of the model. 

\begin{assumption}
\label{assump::noiseless}
Assume that $g$ and $P_\epsilon$ are such that, with $\epsilon_a \equdist \epsilon_b$, $g(f(Z_a), \epsilon_a) \equdist  g(f(Z_b), \epsilon_b)$ if and only if $f(Z_a) \equdist f(Z_b)$.
\end{assumption}

This assumption includes, for example, the noiseless case, and additive noise for a suitable noise distribution \citep[e.g.,][]{Halv2021}. It means that, for identifiability purposes, it is sufficient to analyze the noiseless case. We note that it rules out the possibility of discrete observations except in very limited cases; see \cref{sec:appx:discrete} for a brief discussion of this point. 
Designing a generative model involves specifying parameter spaces for the generator and the prior. We denote these as $\genClass$, a set of measurable mappings $\latent \to \obs$; and $\distClass_z$, a set of probability measures on $\latent$. For example, $\distClass_z$ could be a singleton, as in factor analysis. In ICA, $\distClass_z$ contains only distributions that factorize over its components. 

We will also assume that the generators are bijective on their range, which is required to make the inverse problem of recovering latents well-defined, and is a standard assumption in the identifiability literature. 
\begin{assumption} 
\label{assump::bijective}
Assume that any $f \in \genClass$ is injective, and has the same image: for any $f_a,f_b \in \genClass$, $f_a(\latent) = f_b(\latent) := \genClass(\latent) \subseteq \obs$.
\end{assumption}

\subsection{Model Indeterminacies} \label{sec:id}

A generative model \eqref{eq:model} induces a statistical model as
\begin{align}
\label{eqn::statisticalmodel}
    \obsModel(\genClass,\distClass_z) = \{ P_{\theta} \text{ on } \obs \mid \theta = (f, P_z) \in \genClass \times \distClass_z \} \;.
\end{align}
Classical parameter identifiability can be defined via an equivalence relation $\sim$ on parameter space, $\theta \sim \theta' \iff P_{\theta} = P_{\theta'}$. The equivalence classes of parameters induced by $\sim$ are denoted by $[\theta] := \{ \theta' \colon P_{\theta}=P_{\theta'} \}$, and a model is identifiable up to $[\theta]$. Some authors refer only to the previous case as partial or set identifiability \citep{tamer2010}, and reserve the term identifiability for the case that $[\theta] = \{\theta\}$. We refer to this latter case as strong identifiability. Parameter identifiability does not appear to contain any information about the latent values. We take classical parameter identifiability as our starting point to formulate an alternative definition tailored to latent variable indeterminacy. Specifically, we work with transformations of the latent space that yield different generators but that leave the marginal distribution of the data unchanged.  

\begin{definition} \label{def:indet:trans}
    For a model $\obsModel(\genClass,\distClass_z)$, $A_{a,b}$ is an \textbf{indeterminacy transformation} at $\theta_a, \theta_b$ if $P_{\theta_a} = P_{\theta_b}$ and $f_a \circ A_{a,b}^{-1} \equae f_b$. The \textbf{indeterminacy set} of a model $\obsModel(\genClass,\distClass_z)$, denoted $\indet(\obsModel)$, is the collection of all indeterminacy transformations of the model.
\end{definition}

\Cref{def:indet:trans} describes latent variable indeterminacy; it implies that if $\indet(\obsModel)$ contains non-trivial transformations, any latent $z$ that generates $x = f(z)$ has an equivalent counterpart $A(z)$ for some $A \in \indet(\obsModel)$. 
The identity map on $\latent$, $\id_z$, is always (trivially) an indeterminacy transformation by taking $\theta_a = \theta_b$. A simple way to construct non-trivial candidate indeterminacy transformations at $\theta_a\neq\theta_b$ is by ``pushing forward'' and ``pulling back'' along the generators,
\begin{align} \label{eq:indet:map}
    \genTrans_{a,b}(z) := f_b^{-1}(f_a(z)) \;, \quad z \in \latent \;.
\end{align}
It is possible to show that $\genTrans_{a,b} \in \Aut(\latent)$, and thus if transforming $\latent$ in this way results in $P_{\theta_a} = P_{\theta_b}$, then $\genTrans_{a,b}$ is an indeterminacy transformation.
We refer to $\genTrans_{a,b} \colon \latent \to \latent$ as the \textbf{generator transform} between $\theta_a$ and $\theta_b$. 
It turns out that generator transforms characterize all the possible indeterminacy transformations of a model---though not every generator transform is an indeterminacy transformation. 
\begin{lemma}
	\label{thm:indet:transport}
	Let $\theta_a = (f_a, P_{z,a})$ and $\theta_b = (f_b, P_{z,b})$ be two parametrizations of a generative model with resulting marginal distributions $P_{\theta_a}$ and $P_{\theta_b}$. Then, $P_{\theta_a} = P_{\theta_b}$ if and only if $\genTrans_{a,b}$ is a $(P_{z,a}, P_{z,b})$-measure isomorphism. Furthermore, any $A \in \Aut(\latent)$ is an indeterminacy transformation at $\theta_a, \theta_b$ (\cref{def:indet:trans}) if and only if $A \equae \genTrans_{a,b}$, and therefore all indeterminacy transformations $A_{a,b}$ must be $(P_{z,a}, P_{z,b})$-measure isomorphisms. 
\end{lemma}

\Cref{thm:indet:transport} is the technical foundation for the rest of the paper, and can be interpreted as an existence and uniqueness result, as follows. $\genTrans_{a,b}$ is a distinguished indeterminacy transformation at $\theta_a$, $\theta_b$ that always exists; it is probabilistically equivalent to all other indeterminacy transformations, making it essentially unique. It implies that two parameterizations $\theta_a, \theta_b$ induce the same marginal distribution on observations if and only if one can transport between the two latent measures in $\distClass_z$ by pushing and pulling along the generators. Not only are there no other non-equivalent indeterminacy transformations, but $A_{a,b}$ cannot be an indeterminacy transformation for any other parameters $\theta_c,\theta_d$ unless $\genTrans_{a,b} = f_b^{-1}\circ f_a \equae f_d^{-1} \circ f_c = \genTrans_{c,d}$. Both of these properties follow directly from the injectivity of the generators; relaxing \cref{assump::bijective} would require a substantially different theory than the one developed here.

\subsection{Characterizing Model Identifiability} \label{sec:indet:maps}

Let $\widetilde{\id}_z$ denote the set of all Borel automorphisms that are equal a.e.\ to the identity mapping on $\latent$. If $\indet(\obsModel) = \widetilde{\id}_z$, then $f_a \equae f_b$ for all $\theta_a, \theta_b$ with $P_{\theta_a} = P_{\theta_b}$. We use this to define the appropriate notion of identifiability. 

\begin{definition} \label{def:iden}
	A generative model $\obsModel(\genClass,\distClass_z)$ is \textbf{weakly identifiable up to $\indet(\obsModel)$}, or \textbf{$\indet(\obsModel)$-identifiable}, if its indeterminacy transformations are $\indet(\obsModel)$. If $\indet(\obsModel) = \widetilde{\id}_z$, then the model is \textbf{strongly identifiable}. 
\end{definition}

Strong identifiability means that $f_a(z) = f_b(z)$ for all $z \in \latent$ outside of a set of measure zero. If $f_a,f_b$ are continuous functions, this implies that $f_a(z) = f_b(z)$ for all $z \in \latent$, assuming the reference measure has full support. This definition is in correspondence to the classical definition, in the sense that the equivalence classes $[\theta]$ correspond to subsets of $\indet(\obsModel)$ (see \cref{prop:equiv:rel:properties} for details). Note that when the generator is further parameterized (for example via deep neural networks), this identifiability does not necessarily pass to the generator parameters (network weights). Rather, we use this as a proxy to make the notion of latent variable recoverability precise.

\Cref{thm:indet:transport} gives a necessary and sufficient condition for two parameterizations to correspond to the same marginal distribution on observations. It also documents the cases in which $P_{\theta_a}$ cannot be equal to $P_{\theta_b}$, and allows us to construct the set of model indeterminacies $\indet(\obsModel)$ as indeterminacies generated by $\genClass$ and $\distClass_z$. To that end, define the following subsets of $\Aut(\latent)$ induced by the generative model,
\begin{gather*}
    \indet(\genClass) = \{A \in \Aut(\latent) \mid \exists f_a, f_b \in \genClass \text{ s.t. }A \equae f_b^{-1} \circ f_a \} \\
    \indet(\distClass_z) = \{A \in \Aut(\latent) \mid \exists P_a, P_b \in \distClass_z \text{ s.t. } A_{\#}P_a = P_b \} \;.
\end{gather*}
$\indet(\genClass)$ consists of all possible indeterminacy transforms constructed from $\genClass$, which are a.e. equivalent to the generator transforms. $\indet(\distClass_z)$ consists of all possible isomorphisms between measures in $\distClass_z$. Both sets always include $\widetilde{\id}_z$ by taking $f_a = f_b$ and $P_a = P_b$. The indeterminacies of the model are precisely their intersection. 

\begin{theorem}
    \label{thm:id_general}
    The generative model $\obsModel(\genClass, \distClass_z)$ is identifiable up to $\indet(\obsModel) = \indet(\genClass) \cap \indet(\distClass_z)$. 
    In particular, it is strongly identifiable if and only if $\indet(\genClass) \cap \indet(\distClass_z) = \widetilde{\id}_z$.
\end{theorem}

This expresses the identifiability of a generative model in terms of the indeterminacy transforms induced by its parameter spaces. In particular, \emph{all model indeterminacies must be transports between distributions in $\distClass_z$ that can be constructed by pushing and pulling along generators from $\genClass$ as $f_b^{-1}\circ f_a$}. These are the scaling and permutation matrices in the linear ICA example from \cref{sec:fa:ica}. When $\distClass_z$ is a singleton $\{P_z\}$ as in factor analysis, they are $P_z$-measure-preserving automorphisms, such as the rotation matrices preserving the standard Gaussian.

\subsection{Beyond Linear Generators} 
\Cref{thm:id_general} exposes the structure of unidentifiability and indicates that there may be approaches to specifying strongly identifiable non-linear generative models. It suggests that model identifiability strengthens as we increase the number of constraints on $\genClass$, $\distClass_z$, or both, until the intersection $\indet(\genClass) \cap \indet(\distClass_z)$ contains only the identity. We demonstrate examples of leaving the generator unconstrained and only constraining $\indet(\distClass_z)$, as well as a triangular constraint on $\indet(\genClass)$ in \cref{sec:ME,sec:transports}, respectively.

\section{\uppercase{Task identifiability}} \label{sec:task:id}

In the literature on model identifiability, practical implications are often ignored. How does model identifiability relate to the intended uses of a model after it is fit to data? 
In practice, one may have a model that is only weakly identifiable but that still allows some tasks to be identified. That is, an unidentifiable model may still be useful for some tasks. For example, a causal estimand may still be identified without identification of the full causal model. 
For other tasks, an unidentifiable model is not useful. We make this idea precise before giving some examples from the literature.  

Let $\bfx_m = (x_1,\dotsc,x_m)$ be a finite collection of points in $\bfX$ (e.g., $m$ observations), and recall that $\theta = (f,P_z)$.  
Define a \emph{task} as a pair of functions $(s,t)$: the function $s(\theta, \bfx_m)\in \latent^n$ selects a set of $n$ points in the latent space to use for the task, e.g., estimates of $f^{-1}(x_i)$; and $t \colon \genClass \times \distClass_z \times \bfX^m \times \latent^n \to \bfT$ generates the task output as $t(\theta, \bfx_m, s(\theta,\bfx_m))$. Here, $\bfT$ is the task output space, such as $\bfT = \bfX$ for sample generation. We consider a task to be identifiable if two equivalent model fits produce the same task output.

\begin{definition}  \label{def:task:id}
    A task $(s,t)$ is \textbf{identifiable at $[\theta]$} if, for all $\theta \sim \theta'$ and $\bfx_m \in \bfX^m$,
    \begin{align}
         t(\theta, \bfx_m, s(\theta,\bfx_m)) = t(\theta', \bfx_m, s(\theta',\bfx_m)) \;.
    \end{align} 
    A task is (globally) identifiable if it is identifiable at all $[\theta]$.
\end{definition}

This captures the idea that two different research groups may obtain different but equivalent model fits ($\theta$ and $\theta'$) from the same data, and still reach the same conclusion for a task.  Before proceeding, we give some examples.

\kword{Tasks Without Reference to Observations} 
Consider two labs with different but equivalent models, $\theta_a\sim \theta_b$. Now both labs set $s(\theta, \bfx_m) = c$, without reference to observations. This could, for example, correspond to a `do' intervention when $\latent$ is endowed with a causal model \citep{Yang2021, shen2022, Lu2022}. As the task output, each lab generates a synthetic observation, $f_a(c)$ and $f_b(c)$. Without strong model identifiability, in general $f_a(c) \neq f_b(c)$ and the task is unidentifiable due to the fact that the value $c$ has no inherent meaning in either model. Tasks $t(\theta, \bfx_m, c)$ will not be identifiable in general: the selection of latent points must incorporate the transformation of $\latent$ that occurs when $\theta_a \mapsto \theta_b$; only if the task output is constant across all equivalent $\theta$ will it be identifiable.

\kword{Disentanglement and Latent Shifts} 
A common demonstration of disentanglement is to select a latent variable $s(\theta,x_i) = f^{-1}(x_i)$ corresponding to an observation $x_i$, apply a shift $\delta$ along a latent dimension $k$ (with unit vector $e_k$), and then generate synthetic data $t_{\delta e_k}(\theta,x_i, s(\theta,x_i)) = f(\delta e_k + f^{-1}(x_i))$ for a sequence of different values of $\delta$ \citep[e.g.,][]{higgins2017,Lu2022}. For most models, this task is not identifiable. To see why, observe that task identifiability requires
\begin{align*}
    f_a( & \delta e_k + f_a^{-1}(x) )  = f_b(\delta e_k + f_b^{-1}(x) ) \\
    & = f_a(\genTrans^{-1}_{a,b}( \delta e_k + \genTrans_{a,b}(f_a^{-1}(x)) ) ) \;.
\end{align*}
In general, this will not be the case unless
\begin{align*}
    \genTrans^{-1}_{a,b}( & \delta e_k + \genTrans_{a,b}(f_a^{-1}(x)) = \delta e_k + f_a^{-1}(x) \;.
\end{align*}
That is, $\genTrans_{a,b}$ must commute with $\delta e_k$. If that property is to hold for all $\delta$, $e_k$, and $x$ then $\genTrans_{a,b}$ must be itself a translation.  

\kword{Independence Testing and Causal Discovery} \ 
In applications of (nonlinear) ICA to causal discovery, tests for independence between observed variables and components of latent variables are conducted \citep{monti2020,Khem2020a}. In the simplest case, the goal is to determine the causal direction between two observed variables, $X_1,X_2$. After fitting a model with independent latent components, pair-wise independence tests are conducted; if $X_1$ is a cause of $X_2$ and not vice versa, then $X_1$ will be independent of the second component of the corresponding latent variable, $f^{-1}(X_1,X_2)_2$. The task, based on observations $(x_i)_{i=1}^N =(x_{i,1}, x_{i,2})_{i= 1}^N$ is then $s(\theta_a,(x_i)_{i=1}^N) = (f_a^{-1}(x_i))_{i=1}^N$ and
\begin{align*}
    t(\theta_a, & (x_i)_{i=1}^N, s(\theta_a,(x_i)_{i=1}^N)) \\
    & = \text{IndTest}((x_{i,1})_{i= 1}^N, (f_a^{-1}(x_{i,1},x_{i,2})_2)_{i= 1}^N) \;.
\end{align*}  

If the model is identifiable up to component-wise transformations then $f_a^{-1}(x_{i,1},x_{i,2})_2 = h(f_b^{-1}(x_{i,1},x_{i,2})_2)$, for some function $h \colon \bbR \to \bbR$. Since two real-valued random variables $U,V$ are independent if and only if all transformations $k(U),k'(V)$ are independent, this task remains identifiable under component-wise indeterminacies.\footnote{ICA-type models also have permutation indeterminacy; however, there is a ``correct'' permutation that can be distinguished within causal discovery \citep{shimizu2006}.} 

Define a local subset of indeterminacy transformations as $\indet(\obsModel)|_{\theta} = \{ A \in \indet(\obsModel) \colon P_{\theta} = P_{A\theta}\}$, and $A\theta = (f\circ A^{-1}, A_{\#} P_z)$. Our main result on task identifiability is an application of \cref{prop:equiv:rel:properties}. 

\begin{proposition}
    \label{prop:taskid}
    A task $(s,t)$ is identifiable at $[\theta]$ if and only if, for each $\bfx_m \in \bfX^m$, 
    \begin{align} \label{eq:theta:id}
        t(\theta, \bfx_m, s(\theta, \bfx_m)) & = t(A\theta, \bfx_m, s(A\theta, \bfx_m))  \;,
    \end{align}
    for each $A \in \indet(\obsModel)|_{\theta}$. 
    A sufficient condition for $(s,t)$ to be identifiable at $[\theta]$ is hence if the following holds for each $A \in \indet(\obsModel)|_{\theta}$ and $\bfx_m \in \bfX^m$:
    \begin{gather} \label{eq:equiv:task}
            t(\theta,\bfx_m, \bfz_n) = t(A\theta, \bfx_m, A(\bfz_n)) \\
           \textrm{and} \quad s(A\theta,\bfx_m) = A(s(\theta,\bfx_m)) \;. \nonumber
    \end{gather}
\end{proposition}

Clearly, any task is identifiable if $\obsModel$ is strongly identifiable: $\indet(\obsModel)|_{\theta} = \indet(\obsModel) =  \widetilde{\id}_z$. Weak identifiability may be enough depending on the task, \cref{prop:taskid} gives sufficient conditions based on the symmetries implied by $\indet(\obsModel)$. In the next two sections, we apply the theory of \cref{sec:defs} to obtain strongly identifiable models.

\section{\uppercase{Generative models in multiple environments}} \label{sec:ME}
Suppose data arise from environments indexed by $e \in E$, where $E$ is an arbitrary set, and the environment label is assumed to be deterministic (i.e., known, or observed without noise). Each environment corresponds to a different observation random variable $X^{e}\sim P_x^{e}$ on a shared observation space $\obs$. This is reflected in the generative model as $|E|$ distinct distributions on latent variables, $Z^{e}\sim P_z^{e}$ on a shared latent space $\latent$. Crucially, each environment shares the same generator $f$. We denote the generative model $\obsModel(\genClass, \{\distClass_z^e\}_{e\in E})$ specified as, for each $e \in E$,
\begin{align}
\label{eq:modelme}
Z_i \sim P_z^{e} \;, \  \epsilon_i^{e} \sim P_\epsilon\;,  \  X_i^{e} = g(f(Z_i^{e}), \epsilon_i^{e}) \;,
\end{align}
with $Z_i \condind \epsilon_i^{e}$ in each environment. In general, we do not assume that observations from different environments are paired in any way besides sharing a generator, e.g., they may be from separate datasets.

We use the multiple environments set-up in the same way as auxiliary information in iVAE. The definition of identifiability is subtly different here, but its implications for latent variable indeterminacy remain the same (see \cref{sec:me_proofs} for details). 

\begin{corollary}
    \label{cor:me_id_general}
    The generative model $\obsModel(\genClass,\{\distClass_z^{e}\}_{e \in E})$ is identifiable up to \begin{align}
    \indet(\genClass) \cap \left( \cap_{e \in E} \indet(\distClass_z^{e})  \right).
    \end{align}
\end{corollary}

We note that permutations of the environments are ruled out because the environments are assumed to be known. Clearly, the indeterminacy set shrinks with each environment added, formalizing why auxiliary information has such significant benefits for identifiability. As a demonstration, we recast iVAE under this framework and present the necessary modification for strong identifiability in \cref{sec:ivae,sec:fixedexpfam} (see \cref{proofs:esm} for a similar exercise for the results of \citealt{Ahuja2021}).

\subsection{Multiple Views} 
A result similar to \cref{cor:me_id_general} can be obtained under different modelling assumptions similar to those made by \citet{Gresele2019, Loc2020}. Specifically, we assume here finitely many observed views of each latent variable, indexed by the finite set $E$. We study the following model, denoted by $\obsModel(\{\genClass^{e}\}_{e \in E}, \distClass_z)$: for $e \in E$,
\begin{align}
\label{eq:modelviews}
Z_i \sim P_z \;, \quad  \epsilon_i^{e} \sim P_\epsilon^e \;, \quad X_i^{e} = g^e(f^e(Z_i), \epsilon_i^{e}) \;,
\end{align}
with $Z_i \condind \epsilon^{e}_i$.  Here, for fixed $i$, each observed view $X_i^e$, $e \in E$, is generated by a shared latent variable $Z_i$, and hence the observations are necessarily grouped across environments (unlike the multiple environments setting). Each view $X_i^{e}$ can have a different observation space $\obs^{e}$ with its own generator $f^e \in \genClass^{e}$, e.g., an image and its caption. 

Again, the definition of identifiability here changes, but the implications for latent variable indeterminacy are identical. Intuitively, identification in any of the views yields identification of the generating latent (\cref{sec:mv_proofs}).

\begin{corollary}
    \label{cor:mv_id_general}
    The generative model $\obsModel(\{\genClass^e\}_{e \in E}, \distClass_z)$ is identifiable up to \begin{align} 
    (\cap_{e \in E} \indet(\genClass^e)) \cap \indet(\distClass_z).
    \end{align}
\end{corollary}

This can be useful when, for example, when we have access to a lower-dimensional view (e.g., clinical data) which we might fit using an identifiable triangular flow (as developed in \cref{sec:transports}), in addition to a higher-dimensional view (e.g., genomics data), which can be reconstructed using a VAE. 

\subsection{(Weakly) Identifiable VAE} 
\label{sec:ivae}

In the iVAE model \citep{Khem2020a}, the latent variable distribution varies via an auxillary variable $u$, for example a time index \citep{Hyv2018}. The prior is parameterized as an exponential family density on $\latent = \mathbb{R}^{d_z}$,
\begin{align}
    \label{eq:ivaeprior}
    p(z; \eta(u)) =  m(z)\exp(\eta(u)^\top T(z) - a(\eta(u)),
\end{align}
with functional parameters $\eta$, $T$ taking values in $\bbR^K$. We denote this distribution as $\mathcal{E}_m(\eta(u), T)$. Note that this means the prior is explicitly inferred from the data, contrary to a standard VAE. The remainder of the model design follows \eqref{eq:modelme} with additive noise. We note that  \citet{Khem2020a} assumed the distribution to factorize over dimensions of $Z$ (as in ICA), but that is not necessary, an observation also made recently in \cite{Lu2022}. 

The main identifiability result of \citet{Khem2020a} (their Thm.~1) says that for two parametrizations $(f_a, T_a, \eta_a)$ and $(f_b, T_b, \eta_b)$, if there exist points $u_0, u_1, \dots u_K$ such that $\{\eta_a(u_i) - \eta_a(u_0)\}_{i=1}^{K}$ are linearly independent and span the latent space (and likewise for $\eta_b)$, then there are an invertible matrix $L$ and offset vector $\mathbf{c}$ such that for all $x$,
\begin{align} \label{eq:ivae:indet:orig}
    T_a(f_a^{-1}(x)) = L^{\top} T_b(f_b^{-1}(x)) + \mathbf{c}. 
\end{align}
Our framework allows us to show that this is purely a result of the following property of exponential families. 
\begin{proposition}
    \label{prop:linear_expfam}
    Suppose that $A \in \Aut(\latent)$ is a $(\mathcal{E}_m(\eta_a(u_i), T_a), \mathcal{E}_m(\eta_b(u_i), T_b)$-measure isomorphism for each $i = 0, 1, \dots, K$.
    Suppose that both $\{\eta_a(u_i)\}_{i=0}^K$ and $\{\eta_b(u_i)\}_{i=0}^K$ are linearly independent. Then,
    \begin{align}
        T_b(A(z)) = L^\top T_a(z) + \mathbf{d},
    \end{align}
    almost everywhere, where $L$ is a $K \times K$ invertible matrix and $d$ is a $K$-dimensional vector not depending on $x$.
\end{proposition}

Although the arguments made in the proof of the proposition are similar to those originally presented in \cite{Khem2020a}, our framework indicates that identifiability is a result purely of the exponential family form of the prior---neither diffeomorphic decoders nor independent components are needed. Further assuming that $f$ is diffeomorphic reduces $\indet(\genClass)$, leading to the stronger identifiability results of Theorems 2 and beyond in \citep{Khem2020a}.

\subsection{(Strongly) Identifiable VAE}

\label{sec:fixedexpfam}

In this section, we show how fixed priors, i.e., a nonlinear factor analysis take on VAEs, can lead to strong identifiability. By definition, with fixed latent distribution $P_z$, $\indet(\{P_z\})$ is the set of $P_z$-preserving automorphisms. However, constructing distributions such that the only automorphism is the identity is in general very difficult; non-trivial measure automorphisms, such as via the Darmois construction, almost always exist \citep[see][for an example]{Gresele2021}. On the other hand, we show that multiple environments provide a simple solution to this problem.

We denote a set of full-rank exponential family densities indexed by $\eta \in \bbR^K$ as 
\begin{align} \label{eq:expfam:set}
    \mathcal{E}_{m, T} = \{  p_{\eta}(z) = m(z)\exp{(\eta^\top T(z)-a(\eta))} \mid \eta \in \bbR^K \} \;.
\end{align}

We refer to a particular distribution in such a set as $\mathcal{E}_{m, T}(\eta)$.  Compared to the isomorphisms in \cref{prop:linear_expfam}, there is a stronger result characterizing the automorphisms of collections of exponential families.

\begin{proposition}
    \label{prop:orthogonal}
    Let $\mathcal{E}_{m, T}$ be as in \eqref{eq:expfam:set}, with $m$ strictly positive. Let $\eta_i \in \mathbb{R}^K$ for $i = 0, 1, \dots, K$ and suppose $A \in \Aut(\bbR^{d_z})$ is a $\mathcal{E}_{m,T}(\eta_i)$-measure preserving automorphism for each $\eta_i$. Then, arranging $\eta_i$ into the rows of a matrix $M$, we have
    \begin{align} \label{eq:kernel}
        (T(z) - T(A(z))) \in \textrm{ker}M,\ a.e.
    \end{align}
\end{proposition}

\Cref{tikz_orthogonal} illustrates \eqref{eq:kernel}. 
The proposition provides a characterization that can be specialized in several ways. Firstly, the only \textit{shared} automorphism for each distribution from a suitably fixed exponential family $\mathcal{E}_{m,T}$ whose parameters form a basis of $\bbR^K$ is the identity (under $T$). Strong identifiability follows when these are fixed as the priors. 

\begin{theorem}
    \label{thm:identifiability}
    Let $\obsModel(\genClass, \{\distClass_z^{e}\}_{e \in E})$ be the multiple environments model described in \eqref{eq:modelme}, with $\latent = \bbR^{d_z}$. For a subset of environments $E^* \subset E$, with $|E^*| = K+1$, fix $P_z^{e} = \mathcal{E}_{m,T}(\eta_e)$ with $m$ strictly positive and $T$ injective in at least one dimension, and such that the corresponding parameters $\{\eta_e\}_{e \in E^*}$ span $\mathbb{R}^K$. Then $\obsModel(\genClass, \{\distClass_z^{e}\}_{e \in E})$ is strongly identifiable. 
\end{theorem}

In fact, we do not require a spanning set of environments for a useful characterization of the indeterminacy. A second specialization of \cref{prop:orthogonal} is in the case of Gaussian priors. In such a setting, assuming common covariances across environments, the model is identifiable up to arbitrary transformations on the dimensions \textit{not spanned} by the environment means $\mu_e$; see \cref{sec:geom:indet} for details.

\begin{figure}[bt]
	\centering
		\begin{tikzpicture}[scale=2.2,tdplot_main_coords]

            \coordinate (O) at (0,0,0);
            \coordinate (mu1) at (0.5, -0.2, 0.3);
            \coordinate (mu2) at (-0.4, 0.4, -0.5);
            
            \draw[thick, draw opacity = 0.4, ->] (0,0,0) -- (1,0,0) node[anchor=north east]{$x$};
            \draw[thick, draw opacity = 0.4,->] (0,0,0) -- (0,1,0) node[anchor=north west]{$y$};
            \draw[thick, draw opacity = 0.4,->] (0,0,0) -- (0,0,1) node[anchor=south]{$z$};
            
            
            \draw[thick, color = blue, ->] (O) -- (mu1) node[anchor = south east]{$\eta_1$};
            \draw[thick, color = blue, ->] (O) -- (mu2) node[anchor = south west]{$\eta_2$};
            
            \fill[fill=blue,fill opacity=0.1] ($1.5*(mu1) + 1.5*(mu2)$) -- ($-1.5*(mu1) + 1.5*(mu2)$) -- ($-1.5*(mu1) + -1.5*(mu2)$) -- ($1.5*(mu1) + -1.5*(mu2)$) -- cycle ;
            
            \fill[] ($-1.5*(mu1) + 1.5*(mu2)$) node[color = blue, above left = 8pt] {${\scriptstyle \text{span}(\eta_1, \eta_2)}$};
            \draw[dashed, color = red, <->] ($0.7*(0.166, -1.083, -1)$) -- ($0.7*(-0.166, 1.083, 1)$) node[anchor = south east]{${\scriptstyle T(A(z)) - T(z) \in \text{span}(\eta_1, \eta_2)^\perp}$};
        \end{tikzpicture}
        \caption{The indeterminacy set (under the sufficient statistic) in \textcolor{red}{red} for parameter vectors $\eta_1, \eta_2$.}
        \label{tikz_orthogonal}
\end{figure}
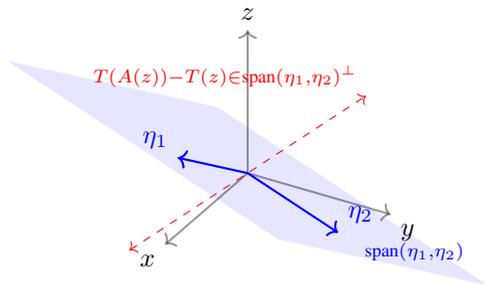

\subsection{Fixed versus Learned Distributions} In ICA compared to factor analysis (\cref{sec:fa:ica}), and now iVAE compared to a strongly identifiable version, learning the latent distributions adds indeterminacies. It is not apparent to us that that source of indeterminacies can be addressed without either: (i) fixing the latent distributions before fitting the generator; or (ii) cleverly constructing $\genClass, \distClass_z$ to ensure that any generator transform $f_b^{-1}\circ f_a$ applied to any $P_z \in \distClass_z$ transports $P_z$ to a different distribution not contained in $\distClass_z$. The latter seems difficult to do in a flexible and general way, though it poses an interesting question for future research. Fixing the latent distributions, on the other hand, is easy, but it comes with questions about flexibility and interpretation. We briefly discuss these here, and note that more work is required. 

Fixing the latent distributions before fitting the generator is one way to achieve strong identifiability, though \textit{how} the distributions ought to be specified depends on the desired use of the model, and the framework for doing so remains unresolved. A task-driven approach, as in \cref{sec:task:id}, is one avenue for exploration. The use of auxiliary data, for example $u$ in iVAE, also appears to be a useful way to structure the latent space. In all of the ``fixed distribution'' identifiability results in this paper, a mapping $u \to \distClass_z$ can be learned from data (as in iVAE), \textit{as long as it is frozen at some point during training}. These frozen parameters should be published alongside any results for reproducibility. We note this does not hinder efforts to generalize to new environments---the fixed environments $E^*$ (as in \cref{thm:identifiability}) may act as ``anchors'' that shape the latent space for the remaining distributions in $E \setminus E^*$.

\section{\uppercase{Identifiability via transports}} \label{sec:transports}

In \cref{sec:fixedexpfam}, $\genClass$ was left unconstrained; strong identifiability was achieved through multiple environments and restricting the environment latent variable distributions to be fixed members of an exponential family. In this section we construct strongly identifiable models in which the latent distributions can be \emph{any} fixed distributions with strictly positive density, and without requiring observations from multiple environments (i.e., auxiliary information). This is achieved by restricting the class of generators, with the additional condition that $\latent=\obs=\bbR^d$. For simplicity of presentation, results are stated for a single environment, but the results also apply in multiple environments or views. 

We approach the problem indirectly, by considering what properties of $\indet(\genClass)$ would yield strong identifiability. We aim to specify $\genClass$ such that when $\genTrans_{a,b} = f_b^{-1}\circ f_a$ transports one latent distribution to another, it is unique in a suitable sense so that, in particular, when it transports a distribution to itself, it must be (equivalent to) the identity map. 

\subsection{Kn\"{o}the--Rosenblatt Transports} 
Triangular monotone increasing (TMI) maps are of growing interest in generative modelling \citep{Kingma2016, Papa2017, Jaini2019, Irons2021}. In particular, they can be used to approximate any fully supported distribution, and can be parametrized via deep neural networks \citep{huang2018,Wehenkel2019}. 
Given two probability measures $P_a$, $P_b$ on $\bbR^d$, there always exists a $(P_a, P_b)$-measure isomorphism with an explicit construction as a TMI map in terms of one dimensional conditional CDF transforms. This is known as the Kn\"{o}the--Rosenblatt (KR) transport. 

The KR transport has several appealing properties. 1) It is the unique (up to a.e.\ equivalence) TMI map that transports between $P_a$ and $P_b$. 2) The class of TMI maps are closed under composition and inversion. 3) Due to its construction, if $P_a = P_b$, then it is almost everywhere equal to the identity map. These properties are ideal for specifying identifiable generative models: 1) gives us a criterion (TMI maps); 2) ensures that the resulting generator transforms are also TMI maps; and 3) says that the resulting measure preserving automorphisms are trivial (akin to triangular matrices in factor analysis). For more details, see \cref{sec:appx:triangular}.

\begin{theorem}
    \label{thm:krmap:id}
	Let $\latent=\obs=\bbR^d$ with fixed latent distribution $P_z$ with strictly positive density. If $\genClass$ is the set of all TMI maps, then $\obsModel(\genClass, \{P_z\})$ is strongly identifiable.
\end{theorem}
This is a generalization of linear factor models with triangular $F$ \citep{Geweke1996, Aguilar2000}. 
Such a generator class also provides identifiability in the ICA sense---we thus believe the following result may also be of independent interest to the ICA community.

\begin{proposition}
    \label{prop:krmap:ica:id}
    Let $\latent=\obs=\bbR^d$. The nonlinear ICA model where $\genClass$ are TMI maps and $\distClass_z$ are fully supported distributions with independent components is identifiable up to invertible, component-wise transformations.
\end{proposition}

\subsection{Optimal Transport Indeterminacies} Though the KR transport is not an optimal transport map itself, it is the limit of optimal transport maps for a sequence of appropriately weighted quadratic costs \citep[][Ch.~2.4]{Santambrogio2015}, and the properties that yield identifiability may apply to some optimal transport maps more generally. 

Given two probability measures $P_a$ and $P_b$ on $\bbR^d$, the Monge formulation of the optimal transport problem with respect to the cost function $c \colon \latent \times \latent \to \bbR_+$ is to find a map $T \colon \latent \to \latent$ such that $T_{\#}P_a = P_b$ and that minimizes the total cost \citep{Santambrogio2015},
\begin{align} \label{eq:OT:cost}
    \int_{\bbR^d} c(z, T(z)) dP_a(z) \;.
\end{align}
We call $T$ an optimal transport (OT) map with respect to $c$ if it minimizes \eqref{eq:OT:cost} for transporting between \textit{some} pair of probability distributions on $\bbR^d$. Let $\calT_c$ be the set of OT maps with respect to a cost $c$. 
Most OT cost functions are derived from metrics, and thus typically $c(z_1,z_2) = 0 \iff z_1 = z_2$, in which case we say that $c$ is separating. This has the implication that the unique OT map from a distribution $P_z$ to itself must be equal to the identity map $P_z$-almost everywhere, as it incurs the minimal cost 0. We use this to formulate a sufficient condition for strong identifiability.  

\begin{theorem} \label{thm:transport:ID}
	Let $\latent=\obs=\bbR^d$, with fixed latent distribution $P_z$ with full support on $\bbR^d$. If $\indet(\genClass) \subseteq \calT_c$ for a separating cost, then $\obsModel(\genClass, \{P_z\})$ is strongly identifiable.
\end{theorem}

We are unaware of any general and flexible function classes that satisfy this property. In particular, properties 1) and 3) that yield strong identifiability in TMI maps are generally true for optimal transport maps, but property 2), the closure property, typically does not hold. Detailed specification of such models is left for future work.

\section{\uppercase{Conclusion}} \label{sec:conclusion}

We have developed a general formal framework for analyzing the sources of indeterminacy in a broad class of generative models. The framework brings seemingly disparate approaches to the problem together: if a model satisfies our assumptions then any identifiability result \emph{must} be a special case of \cref{thm:id_general}. The theory is descriptive rather than constructive, in the sense that proving identifiability results for specific models still requires the non-trivial work of characterizing $\indet(\genClass)$ and $\indet(\distClass_z)$. To that end, \cref{thm:id_general} also makes the sources of indeterminacy visible, enabling more straightforward reasoning during model design. That visibility was crucial to our strong identifiability results, particularly for the transport-based models in Section 5. Those models are far from exhaustive, and we believe that our framework can be useful in developing novel strongly identifiable generative models.

\subsection{Limitations and Societal Implications} 

As with much of the identifiability literature, our results require that the generator is injective with ignorable noise; this is a technical limitation that must be overcome in order to have a theory that is applicable to the full array of models in use. We also note that incorporating inference from finite data into the framework is an important unsolved problem; our framework is limited in that respect. Finally, generative models with stronger identifiability properties can be used to the benefit of society (e.g., latent causal models for drug development), or to its detriment (e.g., better control over harmful synthetic data generation). The present work does not specifically address either, as it is purely theoretical; the potential impact depends on the application.

\subsubsection*{Acknowledgements}
The authors acknowledge helpful comments from anonymous reviewers at AISTATS 2023, as well as the organizers and participants of the \emph{Generative Models and Uncertainty Quantification Workshop} (Copenhagen, September, 2022) for challenging and helpful discussions. BBR gratefully acknowledges the support of the Natural Sciences and Engineering Research Council of Canada (NSERC): RGPIN-2020-04995, RGPAS-2020-00095, DGECR-2020-00343. 

\bibliography{citations.bib}

\appendix
\onecolumn

\section{\uppercase{Indeterminacies and parameter identifiability}}
\label{sec:appx:equivid}
In this section, we justify the study of indeterminacy transformations, which are mappings on $\latent$, as a proxy for identifiability in parameter space. For any $A \in \Aut(\latent)$ and $\theta = (f,P_z) \in \genClass\times\distClass_z$, define $A\theta = (f\circ A^{-1}, A_{\#} P_z)$. We say that $A\theta_a = \theta_b$ if $A_{\#}P_{z,a} = P_{z,b}$ and $f_a \circ A^{-1} \equae f_b$. Finally, let $\indet(\obsModel)|_{\theta} = \{ A \in \indet(\obsModel) \colon P_{\theta} = P_{A\theta}\}$---we can show that this is a local set of indeterminacy transformations ``originating'' at $\theta$ (\cref{lem:uniqueness}). 

\begin{lemma}
    \label{lem:uniqueness}
    $\indet(\obsModel)|_{\theta_a}$ consists exactly the collection of indeterminacy transformations at $\theta_a$ and all $\theta_\argdot \sim \theta_a$.
\end{lemma}

\begin{proof}
    Let $A'$ denote the set of indeterminacy transformations at $\theta_a$ and some arbitrary $\theta_\argdot \sim \theta_a$. We show that $\calA' = \indet(\obsModel)|_{\theta_a}$.

    We first show that $\calA' \subset \indet(\obsModel)|_{\theta_a}$. Let $\theta_\argdot \sim \theta_a$, i.e., $P_{\theta_a} = P_{\theta_\argdot}$. Then $A_{a, \argdot} \in \calA'$ is an indeterminacy transformation at $\theta_a, \theta_\argdot$. By \cref{thm:indet:transport}, $A_{a, \argdot \#}P_{z,a} = P_{z, \argdot}$, and by definition of an indeterminacy transformation, $f_a \circ A_{a, \argdot}^{-1} \equae f_\argdot$. Hence $A_{a,\argdot}\theta_a = \theta\argdot$, which shows that $A_{a, \argdot} \in \indet(\obsModel)|_{\theta_a}$.

    We now show that $\indet(\obsModel)|_{\theta_a} \subset \calA'$. Let $A \in \indet(\obsModel)|_{\theta_a}$ and denote $A\theta = \theta_\argdot$, so that $P_{\theta} = P_{A\theta} = P_{\theta_\argdot}$. By definition of $A\theta$, we have $f_a \circ A^{-1} \equae f_\argdot$. Hence, $A$ is an indeterminacy map at $\theta, \theta_\argdot$, i.e., $A \in \calA'$.   
\end{proof}

We are now ready to prove \cref{prop:equiv:rel:properties}.

\begin{proposition} \label{prop:equiv:rel:properties}
    The equivalence class of each $\theta \in \genClass\times\distClass_z$ is generated by $\indet(\obsModel)|_{\theta}$. That is, $[\theta] = \{A\theta \colon A \in \indet(\obsModel)|_{\theta} \}$.
\end{proposition}

\begin{proof}[Proof of \cref{prop:equiv:rel:properties}]
    
    For notational purposes, we denote an arbitrary $\theta$ by $\theta_a$. 

    We wish to show that $[\theta_a] =  \{A\theta_a \colon A \in \indet(\obsModel)|_{\theta_a} \}$. Note that by \cref{lem:uniqueness}, $\indet(\obsModel)|_{\theta_a}$ consists of indeterminacy transformations at $\theta$ and some arbitrary $\theta_\argdot \sim \theta_a$.

    We first show that $[\theta_a] \subset \{A\theta_a \colon A \in \indet(\obsModel)|_{\theta_a} \}$. Let $\theta_\argdot \in [\theta_a]$. Then, $P_{\theta_a} = P_{\theta_\argdot}$ which by \cref{thm:indet:transport} implies that $\genTrans_{a, \argdot}\theta_a = \theta_\argdot$. Clearly, $\genTrans_{a,\argdot}$ is an indeterminacy transformation at $\theta, \theta_\argdot$, and so $\genTrans \in \indet(\obsModel)_{\theta_a}$, implying that $\theta_\argdot \in \{A\theta_a \colon A \in \indet(\obsModel)|_{\theta_a} \}$.
    
    We now show that $\{A\theta_a \colon A \in \indet(\obsModel)|_{\theta_a} \} \subset [\theta_a]$. Let $\theta_\argdot \in \{A\theta_a \colon A \in \indet(\obsModel)|_{\theta_a} \}$. Then, $\theta_\argdot = A_{a, \argdot}\theta_a$ for some indeterminacy transformation $A_{a, \argdot}$. Since $A_{a, \argdot}$ is an indeterminacy transformation, $P_{\theta_a} = P_{\theta_\argdot}$ by definition and hence $\theta_\argdot \in [\theta_a]$. 
\end{proof}

The above result states that the indeterminacy transforms (mappings on $\latent$) can be used to generate the equivalence classes of parameters. This justifies the study of indeterminacy transformations of $\latent$ (and their extension to the parameter space), rather than studying the parameter space directly, and puts identifiability up to indeterminacy transformations (\cref{def:iden}) in correspondence with parameter identifiability. 

\section{\uppercase{Discrete observations}} 
\label{sec:appx:discrete}
In this section, we discuss models with discrete observations, e.g., Bernoulli with probability parameter given by $f(z)$, or Poisson with mean parameter $f(z)$ (such models were briefly discussed in iVAE \citep{Khem2020a}, as well as in follow-up work such as the pi-VAE \citep{Zhou2020}). In short, the framework developed in this paper rests on bijective generators which enable the recovery of unique latent codes for each observed value. As noted in a correction in \cite{Khem2020a}, this task seems fundamentally impossible for example when the latent space is uncountable and the outcome is discrete, due to the lack of an bijective map between spaces of different cardinality. However, generative models do not typically send a latent variable to the outcome, but rather to a parameter value of a conditional distribution. This allows us to reformulate the assumptions required for our theory, although as we will see shortly, most discrete outcome models do not satisfy these assumptions. 
Formally, suppose $\obs$ is either finite or countable. Let $X$ be a random variable on $\obs$ and denote by $P_x := P(X = x): \obs \to [0,1]$ the probability mass function, which satisfies $\sum_{x \in \obs} P_x(x) = 1$. Two random variables $X_a$, $X_b$ are said to be equal in distribution if and only if their respective PMFs satisfy $P_{x,a}(x) = P_{x,b}(x)$ for all $x \in \obs$.

The observation model is then described by a conditional PMF $P(X = x |z)$. We assume this model has a topological parameter space $\Theta$ and also pair it with the Borel $\sigma$-algebra, e.g., $\Theta = [0,1]^n$ for an $n$-dimensional Bernoulli. Let $f: \latent \to \Theta$ be an injective generator (note this implies $\Theta$ has cardinality at least that of $\latent$). Then the generative model is as follows:
\begin{align}
Z \sim P_z \;, \quad P(X = x | z) = g_{x}(f(z)) \;,
\label{eq:model_discrete}
\end{align}
where $g_{x}(\theta)$ is the PMF of the observation model with parameter $\theta$ at $x$. 

Recall \cref{assump::noiseless} of the main text. We now introduce its discrete analogoue, which would be required for the theory developed in this paper to apply in the discrete formulation. First, note that the marginal PMF on $\obs$ is given as follows:
\begin{align}
    P_{x}(x) = \int_{\latent} g_x(f(z)) P_z(dz) = \int_{\Theta} g_x(\theta) P_z(f^{-1}(d\theta)) = \E_{\theta}\left[ g_x(\theta) \right],
\end{align}
where as a random variable, $\theta = f(Z)$. The assumption is then as follows:
\begin{assumption}[Discrete analogue to main text: \cref{assump::noiseless}]
\label{assump:discrete}
    Assume that $\E_{\theta_a}\left[ g_x(\theta_a) \right] = \E_{\theta_b}\left[ g_x(\theta_b) \right]$ for each $x \in \obs$ if and only if $\theta_a \equdist \theta_b$.
\end{assumption}

In other words, the distribution of $\theta = f(Z)$ must be characterized by the moments $\E[g_x(\theta_a)]$, for each $x \in \obs$. Indeed, for observational equivalence to imply anything about the latent spaces, such an assumption would be needed. However, it appears that this assumption is rarely satisfied for any reasonable models. For example, the Bernoulli observation model with $P(X = 1 | z) = g_1(\theta) = f(z)$, $\Theta = [0,1]$, requires that the distribution of $\theta$ be characterized by just its first moment, $\E[\theta]$. Of course, this is highly unlikely unless very strict restrictions are placed on $f$ and $P_z$, such as $f_\# P_z$ being Gaussian, indicating that identifiability in this case may be restricted to linear generators.   

More generally, the core of the issue remains the cardinality mismatch between $\latent$ and $\obs$. A necessary condition for $\theta_a \equdist \theta_b$ is that $\E[g(\theta_a)] = \E[g(\theta_b)]$ for all bounded continuous $g: \Theta \to \bbR$ (test functions). For $\Theta$ uncountable, there are clearly uncountably many test functions, while in our discrete assumption above, there are countably many test functions at best. Though we do not make this notion precise here, we believe this makes the discrete assupmtion above unlikely to be satisfied, and hence any notion of identifiability, at least under our framework (which we believe to be reasonably general), is highly unlikely for discrete outcomes with uncountable latent spaces.

\section{\uppercase{Definitions and preliminaries}}
\label{sec:appx:definitions}


The rest of this Appendix takes place in the mathematical setting of measure-theoretic probability. This section will review some relevant definitions. Our standard reference here will be \cite{cinlar2011}, though we will sometimes refer to other texts depending on the specific topic \citep{Kechris1995,Schil5118, Boga2007}. 

\subsection{Basic Definitions}

Let $(E, \tau)$ be a topological space with $E$ a set and $\tau$ its collection of open sets.

\begin{definition}[$\sigma$-algebras, {\citealp[Eq. 1.3]{cinlar2011}}]
    A collection $\calE$ of subsets of $E$ is called a $\sigma$-algebra on $E$ if it is closed under complements and countable unions:
    \begin{align}
    B \in \calE \implies E \setminus B \in \calE, \quad B_1, B_2, \dots \in \calE \implies \cup_{n} A_n \in \calE.
    \end{align}
\end{definition}

Note that a $\sigma$-algebra always contains the empty set and $E$ itself. 
The pair $(E, \calE)$ defines a measurable space. The elements of $\calE$ in this context are called measurable sets. When the $\sigma$-algebra is insignificant, or obvious by context, we will simply refer to the space by $E$.

\begin{definition}[Generated $\sigma$-algebras, {\citealp[Sec. 1]{cinlar2011}}]
    The $\sigma$-algebra generated by a collection of subsets $\calE'$, denoted $\sigma(\calE')$ is the smallest $\sigma$-algebra that contains $\calE'$. 
\end{definition}

In this work, we will always work with what is known as the Borel $\sigma$-algebra. 

\begin{definition}[Borel $\sigma$-algebras, {\citealp[Sec. 1]{cinlar2011}}]
    Let $(E, \tau)$ be a topological space. The Borel $\sigma$-algebra of $E$ is generated by the collection of open sets, $\sigma(\tau)$. We denote it by $\borel(E)$.
\end{definition}

An element $B \in \borel(E)$ is then said to be a Borel set. 

Let $(E, \calE)$, $(F, \calF)$ be two measurable spaces, and $f: E \to F$ a mapping between them. The image of $A \subset E$ is defined as
\begin{align}
    f(A) = \{ f(a) \mid a \in A \} \subset F.
\end{align}
Similarly, the preimage of $B \subset F$ is defined as 
\begin{align}
    f^{-1}(B) = \{ x \in E \mid f(x) \in B \} \subset E.
\end{align}
Most mappings we will be concerned with will be assumed to be \emph{measurable}.

\begin{definition}[Measurable Mappings, {\citealp[Sec. 2]{cinlar2011}}]
    A mapping $f: E \to F$ is said to be $(\calE, \calF)$-measurable if $f^{-1}(B) \in \calE$ for each $B \in \calF$.
\end{definition}

If $(E, \calE) = (F, \calF)$, we will refer to $f$ as simply $\calE$-measurable. 

\begin{definition}[Measures, {\citealp[Sec. 3]{cinlar2011}}]
    Given a measurable space $(E, \calE)$, a mapping $\mu: \calE \to [0, \infty]$ is a measure if it satisfies
    \begin{align}
        \mu(\emptyset) = 0, \quad \mu(\cup_{n} A_n) = \sum_{n}\mu(A_n), \quad (A_n) \text{ a disjoint sequence in } \calE.
    \end{align}
    In particular, if $\mu(E) = 1$, then $\mu$ is said to be a probability measure.
\end{definition}

The triplet $(E, \calE, \mu)$ is known as a measure space, and, if $\mu$ is a probability measure, as a probability space. When there is no risk for confusion, we simply identify a measure space by its measure $\mu$. Two measures $\mu$ and $\nu$ on the same measurable space are equal whenever 
\begin{align}
    \mu(B) = \nu(B), \text{ for all } B \in \calE. 
\end{align}

\subsection{Random Variables, Pushforward measures}

\cite[Ch. 2]{cinlar2011} covers probability spaces in depth. Here, we only review the relevant notion pertaining to random variables and their distributions. A random variable $X$ on $(E, \calE)$ is associated to a probability measure $\mu$, called its distribution, defined as\footnote{Technically, we require a background measure space $(\Omega,\calF, P)$, and a random variable is defined as a measurable function $X: \Omega \to E$.}
\begin{align}
     \mu(B)= P(X \in B) , \text{ for all } B \in \calE. 
\end{align}
Random variables $X$, $Y$ defined on the same measurable space with distributions $\mu$, $\nu$ are said to be equal in distribution if $\mu = \nu$ as probability measures, denoted $X \equdist Y$. 

Any $(\calE, \calF)$-measurable function $f: E \to F$ applied to $X$, denoted $f(X)$, defines a random variable on $(F, \calF)$ with distribution $\mu \circ f^{-1}$, where $f^{-1}$ denotes the preimage of $f$ as a set function $\calF \to \calE$. Whenever it is convenient, we will use the more streamlined \textit{pushforward} notation for the distribution of $f(X)$, as follows:
\begin{align}
    f_\# \mu = \mu \circ f^{-1}.
\end{align}

Finally, we define the notion of a pushforward $\sigma$-algebra.

\begin{definition}[Pushforward $\sigma$-algebras, {\citealp[Sec. 2]{cinlar2011}}]
    \label{def:pushsigma}
    Let $(E, \calE)$ be a measurable space, $F$ be a set, and $f$ be a mapping $f: E \to F$. The pushforward $\sigma$-algebra of $f$ is defined as 
    \begin{gather}
    \sigma(f) = \{B \subset F; f^{-1}(B) \in \calE  \}
    \end{gather}
\end{definition}

It is easily shown that $\sigma(f)$ is a $\sigma$-algebra on $F$, and that $f$ is measurable with respect to $\sigma(f)$ \citep[Exercise 2.20]{cinlar2011}. In fact, it is the smallest $\sigma$-algebra that makes $f$ measurable. 
\subsection{Null sets and absolute continuity}

\begin{definition}[Null Sets, {\citealp[Sec. 3]{cinlar2011}}]
    Given a measurable space $(E, \calE)$, a measurable set $A \in \calE$ is said to be \textbf{null} with respect to a measure $\mu$, or $\mu$-null, if $\mu(A) = 0$. 
\end{definition}

The empty set is always a null set by the definition of a measure, but there can be many more null sets, depending on the measure. Any countable union of null sets is again null. For example, the Lebesgue measure $\lebesgue$ on $\bbR$ assigns null measure to a singleton $\{x\}$ for $x \in \bbR$, and so both the set of natural numbers and rational numbers are $\lebesgue$-null sets. 

For a measure space $(E, \calE, \mu)$, most properties that are consequences of measure-theoretic manipulations can only hold $\mu$-\textit{almost everywhere}. This is a weaker notion than a property holding pointwise, and the strength of a result can depend on the measure $\mu$. 

\begin{definition}[Almost Everywhere, {\citealp[Sec. 3]{cinlar2011}}]
    Given a measure space, a property that is stated for $x \in E$ is said to hold $\mu$-almost everywhere if there exists a measurable set $N$ with $\mu(N) = 0$ such that $P$ holds for all $x \in E \setminus N$.
\end{definition}

Furthermore, two measures are often compared with respect to their null sets.

\begin{definition}[Absolute Continuity, {\citealp[p. 31]{cinlar2011}}]
    A measure $\mu$ is said to be absolutely continuous with respect to $\nu$ defined on the same measurable space, denoted $\mu \ll \nu$, if for any $A \in \calE$ such that $\nu(A) = 0$, we also have $\mu(A) = 0$.  
\end{definition}

An equivalent property is given by the Radon-Nikodym theorem.

\begin{theorem}[Radon-Nikodym, {\citealp[Thm. 5.11]{cinlar2011}}]
    \label{thm:rd}
    $\mu \ll \nu$ on $(E, \calE)$ if and only if there exists a measurable function $p: E \to [0, \infty)$, uniquely defined $\nu$-almost everywhere, such that for any measurable set $A \in \calE$ and measurable function $f$, we have
    \begin{align}
        \int_{x \in A} f(x) \mu(dx) = \int_{x \in A} p(x)f(x) \nu(dx).
    \end{align}
    We call $p$ the density of $\mu$ with respect to $\nu$. 
\end{theorem}

Typically, when discussing a probability measure $P$ on $\bbR^d$, $p$ is the density of $P$ with respect to the Lebesgue measure $\lebesgue$, implicitly assuming that $P \ll \lebesgue$. In this work, we will also be working in this context, referring to $p$ as a probability density, unless stated otherwise. 

\begin{definition}[Equivalence, {\citealp[Problem 19.5]{Schil5118}}]
    Two measures $\mu$, $\nu$ are said to be \textbf{equivalent} if $\mu \ll \nu$ and $\nu \ll \mu$.
\end{definition}

Clearly, equivalent measures assign the exact same null sets, and imply the same ``almost everywhere'' statements. There is again an analogous definition in terms of densities.

\begin{restatable}{lemma}{restateequivdensity}[{\citealp[Exercise 19.5]{Schil5118}}]
    If $\mu$ and $\nu$ are two measures defined on the same measurable space, then any density $p$ of one with respect to the other satisfies $p(x) > 0$ $\mu$-almost everywhere (equiv. $\nu$-almost everywhere) if and only if they are equivalent. 
\end{restatable}

When working on Euclidean spaces, we will simply write almost everywhere (or a.e.) to mean $\lebesgue$-almost everywhere---if a probability measure $P_z$ is equivalent to $\lebesgue$, then $P_z$-almost everywhere is equivalent.

Finally, it should be intuitively obvious that if $f$ and $g$ are $\mu$-almost-everywhere equal, then pushing forward $\mu$ by either $f$ or $g$ results in the same measure.

\begin{proposition} 
    \label{prop:equaepushforward}
    Let $(E, \calE, \mu)$ be a measure space and $(F, \calF)$ a measurable space. Let $f, g: E \to F$ be measurable. Suppose $f = g$ $\mu$-almost everywhere. Then, we have $f_{\#}\mu = g_{\#}\mu$ on $\calF$.
\end{proposition}

\begin{proof}
    Let $B \in \calF$, and $\onevec_B$ be the indicator function for $B$. Clearly, $\onevec_B \circ f = \onevec_B \circ g$ $\mu$-almost everywhere. Then, 
    \begin{align}
        f_{\#}\mu(B) = \int \onevec_B(x) \mu(f^{-1}(dx)) = \int \onevec_B \circ f(x) \mu(dx) \\
        = \int \onevec_B \circ g(x) \mu(dx) = \int \onevec_B(x) \mu(g^{-1}(dx)) = g_{\#}\mu(B),
    \end{align}
    since $\int f d\mu = \int g d\mu$ for $f = g$ $\mu$-almost everywhere.
\end{proof}

\subsection{Bijective Mappings, Borel Isomorphisms}

This section reviews relevant facts about Borel isomorphisms, the main reference is \cite[Ch. 15]{Kechris1995}. In this section, we refer to the Borel spaces measurable spaces $(E, \borel(E))$ and $(F, \borel(F))$ simply by $E$ and $F$.

Bijective mappings $f: E \to F$ can enjoy some additionally nice properties. First, it can be easily shown that 
\begin{align}
    f(f^{-1}(A)) = A, \text{ for all } A \subset E.
\end{align}
Furthermore, the image $f(A)$ defines the pre-image of the inverse mapping $f^{-1}$. This is not necessarily true for non-bijective mappings. 

We now define the notion of isomorphism between Borel measurable spaces.
\begin{definition}[Borel Isomorphism, {\citealp[Ch. 10.B]{Kechris1995}}]
    Let $f: E \to F$ be a bijective mapping. If $f$ and $f^{-1}$ are both measurable, it is known as an \textbf{Borel isomorphism}, and $E$, $F$ are said to be isomorphic. If $E = F$, $f$ is called an \textbf{Borel automorphism}. 
\end{definition}

We denote the space of Borel automorphisms of $E$ as $\Aut(E)$.

\begin{lemma}[{\citealp[Corollary 15.2]{Kechris1995}}]
\label{lem:borelimage}
For $f: E \to F$ Borel and injective, $f(B) \in \borel(F)$ for any $B \in \borel(E)$.
\end{lemma}

Since compositions of measurable functions remain measurable, the composition of Borel isomorphisms is again a Borel isomorphism of the appropriate spaces. Borel-measurable bijections are particularly nice to work with---they are automatically Borel isomorphisms.

\begin{lemma}
\label{lem:boreliso}
Let $E$, $F$ be Borel spaces and $f: E \to F$ be bijective. Then, $f$ is measurable if and only if it is a Borel isomorphism.
\end{lemma}

\begin{proof}
We only prove the forward direction, i.e., showing $f^{-1}$ is measurable if $f$ is measurable. The reverse direction is identically proved. By \cite[Theorem 15.1]{Kechris1995}, for $f$ Borel-measurable and injective, $f(B)$ is a Borel set of $F$ for any Borel set $B$ of $E$. Since $f(B)$ defines the pre-image of $f^{-1}$ for any Borel set $B$ of $E$, it immediately follows that $f^{-1}$ is also Borel-measurable.
\end{proof}

Finally, we define the notion of a measure auto/isomorphism. 

\begin{definition}[Measure Isomorphism, {\citealp[Sec 9.2]{Boga2007}}]
Let $(E, \calE, \mu)$ and $(F, \calF, \nu)$ be two Borel measure spaces. A Borel isomorphism $f: E \to F$ is called a $(\mu, \nu)$-measure isomorphism if
\begin{align}
    f_{\#}\mu = \nu, \quad f^{-1}_{\#}\nu = \mu. 
\end{align}
It is called a $\mu$-measure-preserving automorphism if $(E, \calE, \mu) = (F, \calF, \nu)$. 
\end{definition}

Furthermore, by \cref{prop:equaepushforward}, if $f$ is a measure auto/isomorphism and $f = g$ $\mu$-almost everywhere, then $g$ is also a measure auto/isomorphism.

\section{\uppercase{Proofs}}
\label{sec:appx:proofs}

Before proceeding into the proofs of any specific result, we first state a useful technical fact.

\begin{lemma}
    \label{lem:indeterminacyismsbl}
	Let $f_a, f_b \in \genClass$. Then, $\genTrans_{a,b} = f_b^{-1} \circ f_a \in \Aut(\latent)$.
\end{lemma}

\begin{proof}
	The pre-image of $\genTrans_{a,b}$ is $f_a^{-1} \circ f_b$, since $\genClass$ is a family of injective functions. By definition of measurability, it suffices to show that $f_a^{-1}(f_b(B)) \subset \latent$ is still Borel for any Borel set $B \in \borel(\latent)$. By \cref{lem:borelimage}, $f_b(B) \subset \obs$ is Borel. Then, by measurability of $f_a$, it follows that $f_a^{-1}(f_b(B))$ is Borel and hence $\genTrans_{a,b}$ is measurable.
\end{proof}

\subsection{Proof of \cref{thm:indet:transport}}

The proof of \cref{thm:indet:transport} is conceptually simple, but requires some careful book-keeping to make the measure theoretic arguments precise. Specifically, we are required to analyze $\genTrans_{a,b}$, involving $f_{b}^{-1}$, which is only well-defined on $\genClass(\latent)$. As a result, we need to first establish a measurable space on $\genClass(\latent)$.

Recall the definition of a pushforward $\sigma$-algebra of $f$ (\cref{def:pushsigma}). This defines a $\sigma$-algebra on $\genClass(\latent)$, and by construction, $f$ is $(\borel(\latent), \genClass(\latent))$-measurable. Further, all generators in $\genClass$ induce the same pushforward $\sigma$-algebra.

\begin{lemma}
	For $f_a, f_b$ in $\mathcal{F}$, $\sigma(f_a) = \sigma(f_b)$. 
\end{lemma}

\begin{proof}
	To see that $\sigma(f_a) \subset \sigma(f_b)$, suppose $C \in \sigma(f_a)$. By \cref{lem:onlyborel}, $C$ is Borel, which means that $f_b^{-1}(C)$ is Borel by measurability. Hence, $C \in \sigma(f_b)$. We have $\sigma(f_b) \subset \sigma(f_a)$ by the exact same argument, which implies that $\sigma(f_a) = \sigma(f_b)$.
\end{proof}

We denote this shared $\sigma$-algebra by $\sigma(\genClass)$. Importantly, $\sigma(\genClass)$ is a subset of the Borel sets of $\obs$. 

\begin{lemma}
	\label{lem:onlyborel}
	$\sigma(\genClass)$ contains only Borel sets. In other words, $\sigma(\genClass) \subset \borel(\obs)$.
\end{lemma}

\begin{proof}
	Let $C \in \sigma(\genClass)$, and let $f \in \genClass$ be any generator. By definition, $f^{-1}(C)$ is Borel. Since $f$ is injective, we have $f(f^{-1}(C)) = C$. By \cref{lem:borelimage}, $C$ must be Borel.
\end{proof}

We are now ready to construct the measurable space $(\mathcal{F}(\latent), \sigma(\mathcal{F}))$. Note the following facts about $(\mathcal{F}(\latent), \sigma(\mathcal{F}))$:

\begin{itemize}
	\item For any $f \in \mathcal{F}$, $f: \latent \to \mathcal{F}(\latent)$ is bijective, and $f^{-1} : \mathcal{F}(\latent) \to \latent$ is well defined.
	\item For any $f \in \mathcal{F}$ and a Borel set $B \in \borel(\latent)$, its image $f(B) \subset \mathcal{F}(\latent)$ is also the pre-image of $f^{-1}$---that is, $(f^{-1})^{-1}(B) = f(B)$. 
	\item Since $\sigma(\mathcal{F}) \subset \borel(\obs)$, if any measures are equal on $\borel(\obs)$, then they are also equal on $\sigma(\mathcal{F})$. 
\end{itemize}

We are now ready to prove \cref{thm:indet:transport}. Recall that to say $A_{a,b}$ is a $(P_{z,a}, P_{z,b})$-measure isomorphism is to say that $A_{a,b \#}P_{z,a} = P_{z,b}$---equivalently, for each $B \in \borel(\latent)$, $P_{z,a}(A_{a,b}^{-1}(B)) = P_{z,b}(B)$.

\begin{proof}[Proof of \cref{thm:indet:transport}]~
    We first show the claim that $P_{\theta_a} = P_{\theta_b}$ if and only if $\genTrans_{a,b}$ is a $(P_{z,a}, P_{z,b})$-measure isomorphism.
    \begin{enumerate}
        \item[$\Rightarrow$] Recall that $\genTrans_{a,b} \in \Aut(\latent)$, by \cref{lem:indeterminacyismsbl}. By \cref{assump::noiseless} in the main text, $P_{\theta_a} = P_{\theta_b}$ implies that $f_{a\#}P_{z,a} = f_{b\#}P_{z,b}$ as measures on $\borel(\obs)$. By \cref{lem:onlyborel}, this implies $f_{a\#}P_{z,a} = f_{b\#}P_{z,b}$ also as measures on $\sigma(\mathcal{F})$. Now, let $B \in \borel(\latent)$. Then,
        \begin{align}
            P_{z,a}(\genTrans_{a,b}^{-1}(B)) = P_{z,a}(f_a^{-1}(f_b(B)) = P_{z,b}(f_b^{-1}(f_b(B)) = P_{z,b}(B),
        \end{align}
        where the first equality is by definition (working on $\sigma(\mathcal{F})$), the second equality is due to $f_{a\#}P_{z,a} = f_{b\#}P_{z,b}$, and the third equality is due to injectivity  (on the measurable space $(\genClass(\latent), \sigma(\genClass)$). Since $B$ was arbitrary, this shows that $P_{z,a} \circ \genTrans_{a,b}^{-1} = P_{z,b}$.
        \item[$\Leftarrow$] We show the contrapositive statement, i.e.,
        \begin{align}
            \genTrans_{a,b \#}P_{z,a} \neq P_{z,b} \implies P_{\theta_a} \neq P_{\theta_b}
        \end{align}
        Note that by \cref{assump::noiseless}, the hypothesis is equivalent to $f_{a\#}P_{z,a} \neq f_{b\#}P_{z,b}$. This means that there exists $B \in \borel(\obs)$ such that
        \begin{align}
            P_{z,a}(f_a^{-1}(B)) \neq P_{z,b}(f_b^{-1}(B)).
        \end{align}
        To show that $\genTrans_{a,b \#}P_{z,a} \neq P_{z,b}$ is to find $B^* \in \borel(\latent)$ such that 
        \begin{align}
            P_{z,a}(f_a^{-1}(f_b(B^*)) \neq P_{z,b}(B^*).
        \end{align}
        Now, $f_b^{-1}(B) \in \borel(\latent)$ by measurability, and hence taking $B^* = f_b^{-1}(B)$, we obtain
        \begin{align}
            P_{z,a}(f_a^{-1}(f_b(f_b^{-1}(B))) = P_{z,a}(f_a^{-1}(B)) \neq 
            P_{z,b}(f_b^{-1}(B)),
        \end{align}
        and hence $\genTrans_{a,b \#}P_{z,a} \neq P_{z,b}$, as required.

        Now, note that by definition of $\genTrans_{a,b}$, it is the unique such map that $f_a (\genTrans_{a,b}^{-1}(z)) = f_b(z)$ pointwise. Therefore, all possible indeterminacy transformations must be equivalent almost everywhere to $\genTrans_{a,b}$. By \cref{prop:equaepushforward}, almost everywhere equivalence preserves measure isomorphisms, and hence any indeterminacy transformation $A_{a,b}$ must be a $P_{z,a}, P_{z,b})$-measure isomorphism.
    \end{enumerate}
\end{proof}

\subsection{Proof of \cref{thm:id_general}}

As a result of \cref{thm:indet:transport}, the proof of \cref{thm:id_general} is straight forward.

\begin{proof}[Proof of \cref{thm:id_general}]
    Recall that, for the generative model to be identifiable up to a set of measurable functions $\indet(\obsModel)$ is to say that, for all $(f_a, P_{z,a})$, $(f_b, P_{z,b}) \in \mathcal{F} \times \mathcal{P}_z$ such that $P_{\theta_a} = P_{\theta_b}$, there exists some $A \in \indet(\obsModel)$ such that $A \equae f_b^{-1} \circ f_a$.
    
    We first show that for any parameter spaces $\genClass$ and $\distClass_z$, we have that $\indet(\obsModel) \subseteq \indet(\genClass) \cap \indet(\distClass_z)$. Suppose $A \in \mathcal{A}(\calM)$. That is, $A \equae f_b^{-1} \circ f_a$ for some $P_{z,a}, P_{z,b}$ such that $P_{\theta_a} = P_{\theta_b}$, where $\theta_a = (f_a, P_{z,a}), \theta_b = (f_b, P_{z,b})$. By definition of $A$, we have $A \in \indet(\genClass)$. By \cref{thm:indet:transport}, we must have that $A \in \indet(\distClass_z)$ also.
    
    We now show that $\indet(\genClass) \cap \indet(\distClass_z) \subseteq \indet(\obsModel)$. Suppose $A \in \indet(\genClass) \cap \indet(\distClass_z)$. Since $A \in \indet(\genClass)$, we can write $A \equae f_b^{-1} \circ f_a$ for some $f_a, f_b \in \genClass$, e.g., $A \equae \genTrans_{a,b}$. Since $A \in \indet(\distClass_z)$, there exist $P_{z,a}$ and $P_{z,b}$ such that $A_{\#}P_{z,a} = P_{z,b}$. By \cref{thm:indet:transport}, $\theta_a = (f_a, P_{z,a})$, $\theta_b = (f_b, P_{z,b})$ are such that $P_{\theta_a} = P_{\theta_b}$, and hence $A \in \mathcal{A}(\calM)$. 
\end{proof}

\subsection{Proof of \cref{prop:taskid}}

\begin{proof}[Proof of \cref{prop:taskid}]
    Equation \eqref{eq:theta:id} is just the definition of task identifiability (\cref{def:task:id}) with $\theta' = A\theta$, along with $[\theta] = \{ A \theta \colon A \in \indet(\obsModel)|_{\theta} \}$ from \cref{prop:equiv:rel:properties}. Now, assume that both equations in \eqref{eq:equiv:task} hold. Then
    \begin{align*}
        t(A\theta, \bfx_m, s(A\theta, \bfx_m)) & = t(A\theta, \bfx_m, A(s(\theta,\bfx_m))) \\
        & = t(\theta, \bfx_m, s(\theta, \bfx_m)) \;.
    \end{align*}
\end{proof}

\subsection{Proof of \cref{cor:me_id_general}}
\label{sec:me_proofs}

We first define an appropriate notion of identifiability in the multiple environments model indexed by $e \in E$, $\obsModel(\genClass, \{\distClass_z^e\}_{e\in E})$. Since the environments are known deterministically, we view the overall model as a statistical model over each environment with parameter $\theta^e$. We have $\theta^e = (f, P_z^e) \in \genClass \times \distClass_z^{e}$, where $\distClass_z^{e}$ is not necessarily the same for each $e \in E$. We will denote $\theta = \{\theta^e\}_{e \in E}$, and denote the marginal distribution over $\obs$ in environment $e$ as $P_{\theta}^{e}$ (recall that each of these are over $\obs$). We adapt the definition of an indeterminacy transform as follows. 

\begin{definition}[Multiple Environments Analogue to Main Text: \cref{def:indet:trans}]
\label{def:indet:trans:me}
$A_{a,b}$ is an \textbf{indeterminacy transformation} at $\theta_a, \theta_b$ if $P_{\theta_a}^{e} = P_{\theta_b}^{e}$ for each $e \in E$ and $f_a \circ A_{a,b}^{-1} \equae f_b$.
\end{definition}

That is, $A_{a,b}$ should be an indeterminacy transformation for the model in each environment. The indeterminacy set is again the collection of all indeterminacy transforms, and weak and strong identifiability remain as in \cref{def:indet:trans} and \cref{def:iden} in this setting. Since there are more constraints on $\theta_a$ and $\theta_b$ compared to the single environment case, it should be intuitively clear that the indeterminacy set is smaller in this case. Since the generator and hence the generator transforms do not change across environments, \cref{thm:id_general} applies in each environment. This forms the basic idea of the following proof of \cref{cor:me_id_general}.

\begin{proof}[Proof of \cref{cor:me_id_general}]
    Fix an environment $e \in E$. Recall that an indeterminacy transformation at $\theta_a$, $\theta_b$ in each environment is an automorphism $A_{a,b} \in \Aut(\latent)$ such that $P_{\theta_a}^e = P_{\theta_b}^e$ and $f_a \circ A_{a,b}^{-1} \equae f_b$. By \cref{thm:id_general}, any $A_{a,b}$ satisfying the above lies in $\mathcal{A}(\mathcal{F}) \cap \mathcal{A}(\mathcal{P}_z^e)$. 
    
    Now, by definition, the indeterminacy transformations as defined in \cref{def:indet:trans:me} are necessarily also indeterminacy transformations for each environment $e \in E$. Hence, for any indeterminacy transformation $A_{a,b}$, we have
    \begin{gather}
        A_{a,b} \in \bigcap_{e \in E} \left( \mathcal{A}(\mathcal{F}) \cap \mathcal{A}(\mathcal{P}_z^e) \right) =         \mathcal{A}(\mathcal{F}) \cap \left( \cap_{e} \mathcal{A}(\mathcal{P}_z^e)  \right),
    \end{gather}
    and hence the indeterminacy set is at most $ \mathcal{A}(\mathcal{F}) \cap \left( \cap_{e} \mathcal{A}(\mathcal{P}_z^e)  \right)$. 
\end{proof}

\subsection{Proof of \cref{cor:mv_id_general}}
\label{sec:mv_proofs}

In the multiple environments setting, the model could be viewed $|E|$ separate models which jointly identify the generator. In the multiple views case, $\obsModel(\{\genClass^e\}_{e \in E}, \distClass_z)$, we assume $E = (e_1, \dots, e_n)$ and combine all views into one unifying model, on which we apply our previously developed theory. 

Suppose we have the observation spaces $\obs^{e_i}$ and generator classes $\genClass^{e_i}$ (all of which are injective), which may be different across views (indeed, this should be the case for the result to be interesting). Our model is then parametrized in each view by $\theta^{e_i} = (f^{e_i}, P_z)$, where $f^{e_i} \in \genClass^{e_i}$ and $P_z \in \distClass_z$. 

Note that each observation $\{x_{e_i}\}_{i=1}^{n}$ is assumed to be generated by the same latent point $z$, i.e., the noiseless views are $\{f^{e_i}(z)\}_{i=1}^n$. We can equivalently state this assumption by stacking the views into a multivariate function:
\begin{align}
    f(z) = \begin{bmatrix}
    f^{e_1}(z) \\
    \vdots \\
    f^{e_n}(z)
    \end{bmatrix},
\end{align}
where we have $f: \latent \to \bigtimes_{i = 1}^n \obs^{e_i}$. Note the image of $f$ is significantly smaller than the product of the respective images. Writing $\genClass$ to denote the implied parameter space of the stacked $f$, we can write the shared image as:
\begin{align}
    \genClass(\latent) = \{(x_1, \dots, x_n) \in \bigtimes_{i = 1}^n \obs^{e_i} \mid (f^{e_i})^{-1}(x_i) = (f^{e_j})^{-1}(x_j), \text{ for all } i, j = 1, \dots, n \}.
\end{align}
As an example, consider $f: \bbR \to \bbR^2$ defined by $f(z) = (z, z^3)$. Writing $x = z, y = z^3$, the image is defined by the graph of the function $y = x^3$ in $\bbR^2$. Now, $f$ is invertible as a mapping $\latent \to \genClass(\latent)$, and we have 
\begin{align}
    f^{-1}(x_1, \dots, x_n) = (f^{e_1})^{-1}(x_1) = \cdots = (f^{e_n})^{-1}(x_n).
\end{align}

Clearly, identifying $f$ in the usual sense is equivalent to identifying each underlying generator. Hence, we follow \cref{def:indet:trans} and \cref{def:iden} in the main text, and do not need any additional definitions. For an $f \in \genClass$, $f_\#P_z$ is the joint distribution on $\bigtimes_{i = 1}^n \obs^{e_i}$ with marginals $f^{e_i}_\# P_z$.\footnote{It is sufficient for this section to characterize $f_\#P_z$ up to a coupling of its marginals.} Now, let $\theta = (f = \{f^{e}\}_{e \in E}, P_z)$, we denote the resulting distribution over $\bigtimes_{i = 1}^n \obs^{e_i}$ by $P_{\theta}$, and the corresponding marginals by $P_{\theta}^{e}$. Note that if $\theta_a$ and $\theta_b$ result in the same joint distribution, then their corresponding marginals match also. We now prove \cref{cor:mv_id_general} (for the case $E = (e_1, \dots, e_n)$).

\begin{proof}[Proof of \cref{cor:mv_id_general}]
Suppose $A_{a,b}$ is an indeterminacy transform at $\theta_a$, $\theta_b$. By \cref{thm:indet:transport}, $A_{a,b} \equae \genTrans_{a,b}$. The generator transform $\genTrans_{a,b} = f_b^{-1} \circ f_a$ is given by, for any $z \in \latent$
\begin{align}
    f_b^{-1}(f_a^{e_1}(z), \dots, f_a^{e_n}(z)) = (f_b^{e_1})^{-1}(f_a^{e_1}(x_1)) = \cdots = (f_b^{e_n})^{-1}(f^{e_n}_a(x_n)).
\end{align}
In other words, letting $\genTrans_{a,b}^e$ denote the generator transform for each view, we have
\begin{align}
    \genTrans_{a,b} = \genTrans_{a,b}^{e_1} = \cdots = \genTrans_{a,b}^{e_n}.
\end{align}
Since the marginals match at $\theta_a, \theta_b$ also, $A_{a,b}^{e_i} \equae \genTrans_{a,b}^{e_i} = \genTrans_{a,b} \equae A_{a,b}$ for any indeterminacy transform in each view $e_i$, i.e., since the generator transforms characterize the indeterminacy transforms, $A_{a,b}$ must be an indeterminacy transformation for each $e_i$. By \cref{thm:id_general}, we have $A_{a,b}^{e_i} \in \indet(\genClass^{e_i})) \cap \indet(\distClass_z)$ for each $e_i$, which implies that
\begin{align}
 A_{a,b} \in (\bigcap_{e \in E} \left[ \indet(\genClass^e)) \cap \indet(\distClass_z)\right] = (\cap_{e\in E}\indet(\genClass^e) \cap \indet(\distClass_z)).
\end{align}
\end{proof}

\subsection{Proof of \cref{prop:linear_expfam}}

This section makes heavy use of probability densities and absolute continuity. Refer to \cref{sec:appx:definitions} for precise definitions. We will reproduce the following Lemma for convenience.

\restateequivdensity*

Now, we may state an intermediate result.

\begin{lemma}
    \label{lemma:rd}
    Suppose probability measures $P_{z,a}, P_{z,b}$ admits strictly positive densities $p_a$, $p_b$. Suppose $A$ is a ($P_{z,a}$, $P_{z,b})$-measure isomorphism. Then, 
    \begin{gather}
        p_b(A(x))k_{A}(x) = p_a(x) \quad a.e.,
    \end{gather}
    where $k_{A}$ depends only on $A$ and is strictly positive a.e..
\end{lemma}

\begin{proof}
    Since $A$ is a ($P_{z,a}$, $P_{z,b})$-measure isomorphism and $P_{z,a}, P_{z,b}$ are equivalent to $\lambda_z$, we have that for a Borel set $B$,
    \begin{gather}
        \lambda_z(B) = 0 \iff P_{z,a}(B) = 0 \iff P_{z,b}(A(B)) = 0 \iff \lambda_z(A(B)) = 0,
    \end{gather}
    where the first and third equivalences are because $P_{z,a}$ and $P_{z,b}$ are equivalent to $\lambda_z$. This shows that $\lambda_z \circ A$ is equivalent to $\lambda_z$, and hence it has an a.e.-strictly positive density $k_A$. Then, by the definition of the density (\cref{thm:rd}), we have for a Borel set B,
    \begin{gather}
        P_{z,a}(B) = P_{z,b}(A(B))  \\ \iff \int_{B} p_a(x)\lambda_z(dx) = \int_{A(B)} p_b(x)\lambda_z(dx) = \int_{B} p_b(A(x))\lambda_z(A(dx)),
    \end{gather}
    where the last equality is by the standard change of variables formula, noting that $B = A^{-1}(A(B))$ since $A$ is invertible. Now, we have that
    \begin{gather}
        \int_{B} p_a(x)\lambda_z(dx) = \int_{B} p_b(A(x))k_A(x)\lambda_z(dx),
    \end{gather}
    by invoking the definition of the density again. Since the above holds for any $B$, we have
    \begin{gather}
        p_b(A(x))k_A(x) = p_a(x) \quad a.e.,
    \end{gather}
    where $k_A(x)$ is strictly positive a.e..
\end{proof}

\begin{corollary}
    \label{cor:rd}
    Suppose four probability measures $P_{1,a}, P_{2,a}, P_{1,b}, P_{2,b}$ have strictly positive densities $p_{1,a}, p_{2,a}, p_{1,b}, p_{2,b}$. For $A$ both a $(P_{1,a}, P_{1,b})$-measure isomorphism and a $(P_{2,a}, P_{2,b})$-measure isomorphism, we have
    \begin{gather}
        \frac{p_{1,a}}{p_{2,a}}(x) = \frac{p_{1,b}}{p_{2,b}}(A(x)) \quad a.e..
    \end{gather}
\end{corollary}

\begin{proof}
This follows immediately from \cref{lemma:rd} from the fact that $k_A$ is strictly positive a.e. and depends only on $A$.
\end{proof}

We are now ready to prove \cref{prop:linear_expfam}.

\begin{proof}[Proof of \cref{prop:linear_expfam}]
    Let $A \in \Aut(\latent)$ be $(\mathcal{E}_m(\eta_a(u_i), T_a), \mathcal{E}_m(\eta_b(u_i), T_b)$-measure isomorphisms for all $i = 0, 1, \dots, K$. Suppose that both $\{\eta_a(u_i)\}_{i=0}^K$ and $\{\eta_b(u_i)\}_{i=0}^K$ are linearly independent. Fix $j$ arbitrarily and note that $\{\eta_a(u_i) - \eta_a(u_j) \}_{i \neq j}$ and $\{\eta_b(u_i) - \eta_b(u_j) \}_{i \neq j}$ both still span $\bbR^K$. From \cref{cor:rd} and by taking logarithms, we have for each $i \neq j$,
    \begin{gather}
        \eta_a(u_i)^\top T_a(z) - a(\eta_a(u_i)) - (\eta_a(u_j)^\top T_a(z) - a(\eta_a(u_j))) \\
        = \eta_b(u_i)^\top T_b(A(z)) - a(\eta_b(u_i)) - (\eta_b(u_j)^\top T_b(A(z)) - a(\eta_b(u_j))),
    \end{gather}
    almost everywhere, which simplifies to
    \begin{gather}
        (\eta_a(u_i) - \eta_a(u_j))^\top T_a(z) - c_a(u_i) = (\eta_b(u_i) - \eta_b(u_j))^\top T_b(A(z)) - c_b(u_i),
    \end{gather}
    almost everywhere. $c_a$, $c_b$ are differences in the normalizing constants $a(\eta_a)$, and do not depend on $z$---we suppress the dependency on $u$ for convenience. Written in matrix form, we have
    \begin{gather}
        \begin{bmatrix}
            \eta_a(u_0) - \eta_a(u_j)\\
            \vdots \\
            \eta_a(u_K) - \eta_a(u_j)
        \end{bmatrix}^\top 
        T_a(z) = 
        \begin{bmatrix}
            \eta_b(u_0) - \eta_b(u_j)\\
            \vdots \\
            \eta_b(u_K) - \eta_b(u_j)
        \end{bmatrix}^\top T_b(A(z))
        + \mathbf{c},
    \end{gather}
    almost everywhere, where $\mathbf{c}$ is the vector of differences $c_a - c_b$. Following \citep{Khem2020a}, we will call these two matrices $L_a$ and $L_b$, noting that they are invertible since their rows are linearly independent by assumption. Then, we obtain
    \begin{gather}
        L_a^\top T_a(z) = L_b^\top T_b(A(z)) + c \\
        \implies T_b(A(z)) = (L_b^{-1}L_a)^\top T_a(z) - (L_b^{-1}L_a)^\top \mathbf{c} \\
        \implies T_b(A(z)) = L^\top T_a(z) + \mathbf{d},
    \end{gather}
    almost everywhere, where $L = L_b^{-1}L_a$ is invertible and $\mathbf{d} = -L^\top \mathbf{c}$.
\end{proof}

For completeness, we will state the iVAE identifiability result within these terms. 

\begin{proposition}
    Suppose a generative model is described by \eqref{eq:modelme} with latent distributions described by \eqref{eq:ivaeprior}, and that $m$ is strictly positive. Suppose we observe at least $K+1$ distinct values of $u_i$ such that the corresponding natural parameters $\{\eta(u_i)\}_{i=0}^{K}$ are linearly independent. Then, any indeterminacy transformation satisfies
    \begin{align}
        T_b(A_{a,b}(x)) = L^\top T_a(x) + \mathbf{d},
    \end{align}
    almost everywhere, where $L$ is an invertible $K \times K$ matrix and $\mathbf{d}$ is a $K$-dimensional vector. 
\end{proposition}

\begin{proof}
    By \cref{cor:me_id_general} and since we do not constrain $\mathcal{F}$, the generator is identifiable up to the transformations described in \cref{prop:linear_expfam}. 
\end{proof}

\subsection{Proof of \cref{prop:orthogonal}}

Here, we present proofs for the automorphism version of \cref{prop:linear_expfam}. These results are simply special cases of the theory developed above by setting $P = P_{z,a} = P_{z,b}$ and where $(P_{z,a}, P_{z,b})$-measure isomorphisms are replaced with $P$-measure automorphisms. We begin with an intermediate result.

\begin{lemma}
    \label{lem:linear}
    Let $\mathcal{E}_{m, T}$ be as in \eqref{eq:expfam:set}, with $m$ strictly positive. Suppose $A$ is simultaneously a $\calE_{m,T}(\eta_1)$ and $\calE_{m,T}(\eta_2)$-measure preserving automorphism. Then, we have
    \begin{gather}
        (\eta_1 - \eta_2)^\top T(z) = (\eta_1 - \eta_2)^\top T(A(z)) \quad a.e..
    \end{gather}
\end{lemma}

\begin{proof}
    Denote the densities for $\calE_{m,T}(\eta_1)$ and $\calE_{m,T}(\eta_2)$ as $p_1$, $p_2$. The expression is a direct consequence of \cref{cor:rd} by plugging in the exponential family densities $p_{1,a} = p_{1,b} = p_1$ and likewise for $p_2$. Taking logarithms on both sides, we have
    \begin{gather}
        \eta_1^\top T(z) - \eta_2^\top T(z) - a(\eta_1) + a(\eta_2) = \eta_1^\top T(A(z)) - \eta_2^\top T(A(z)) - a(\eta_1) + a(\eta_2) \\
        \implies (\eta_1 - \eta_2)^\top T(z) = (\eta_1 - \eta_2)^\top T(A(z)) \quad a.e..
    \end{gather}
\end{proof}

This result implies \cref{prop:orthogonal}.

\begin{proof}[Proof of \cref{prop:orthogonal}]
\cref{lem:linear} applies to the $K$ contrast vectors $(\eta_i - \eta_0)$, so we have for each $i = 1, \dots, K$,:
\begin{gather}
    (\eta_i - \eta_0)^\top (T(z) - T(A(z)) = 0 \quad a.e..
\end{gather}
Any vector $v \in \text{span}\{(\eta_i - \eta_0)\}_{i = 1}^{K}$ is of the form $v = \sum_{i=1}^K a_i (\eta_i - \eta_0)$. Clearly, 
\begin{gather}
    v^\top (T(z) - T(A(z)) = 0 \quad a.e..
\end{gather}
Since $\text{span}\{(\eta_i - \eta_0)\}_{i = 1}^{K}$ is the row space of $M$, we have $(T(z) - T(A(z))) \in \textrm{ker}M$, almost everywhere.  
\end{proof}

\subsection{Proof of \cref{thm:identifiability}}

A more useful adaptation of \cref{prop:orthogonal} for strong identifiability is the following result.

\begin{proposition}
    \label{prop:identity_expfam}
    Let $\mathcal{E}_{m, T}$ be a fixed exponential family, with dimension $K$, such that $m$ is strictly positive and $T$ is injective. Suppose that $\eta_i \in \mathbb{R}^K$, $i = 0, 1, \dots, K$ span $\mathbb{R}^K$. Suppose $A: \mathbb{R}^d \to \mathbb{R}^d$ is a simultaneously a $\mathcal{E}_{m,T}(\eta_i)$-measure automorphism for each $\eta_i$. Then, $A(z) = z$, almost everywhere.
\end{proposition}

\begin{proof}
Without loss of generality, assume that $\eta_0$ is such that $\{\eta_i - \eta_0\}$ forms a basis of $\mathbb{R}^K$. By \cref{prop:orthogonal}, $(T(z) - T(A(z))) \in \textrm{ker}M = \{0\}$, since $M$ is full rank. This shows that $T(A(z)) = T(z)$ almost everywhere. If $T$ is injective, then we have 
\begin{gather}
    A(z) = z \quad a.e..
\end{gather}
\end{proof}

Strong identifiability of our proposed method then follows.

\begin{proof}[Proof of \cref{thm:identifiability}]
    In this model, we have $\distClass_z^{e} = \{\mathcal{E}_{m,T}(\eta_e)\}$. By \cref{cor:me_id_general}, the generator is identifiable up to $\mathcal{A}(\mathcal{F}) \cap \left( \cap_e \mathcal{A}(\{\mathcal{E}_{m,T}(\eta_e)\})  \right)$.
    By \cref{prop:identity_expfam}, $\cap_e \mathcal{A}(\{\mathcal{E}_{m,T}(\eta_e)\})$ contains only functions that are equal to the identity almost everywhere, and hence the model is strongly identifiable.
\end{proof}

\subsection{Geometric characterization of iVAE indeterminacies} \label{sec:geom:indet}

We can further characterize the indeterminacy set geometrically in terms of the orthogonal complements of the subspace spanned by the natural parameter vectors; this is the Gaussian specialization of \cref{prop:orthogonal}. See \cref{tikz_orthogonal_appx} for a visualization.

\begin{figure}[b]
	\centering
		\begin{tikzpicture}[scale=2.3,tdplot_main_coords]
            \coordinate (O) at (0,0,0);
            \coordinate (mu1) at (0.5,0,0);
            \coordinate (mu2) at (0, 0.5,0);
            \draw[thick, draw opacity = 0.4,->] (0,0,0) -- (1,0,0) node[anchor=north east]{$x$};
            \draw[thick, draw opacity = 0.4,->] (0,0,0) -- (0,1,0) node[anchor=north west]{$y$};
            \draw[thick, draw opacity = 0.4,->] (0,0,0) -- (0,0,1) node[anchor=south]{$z$};
            \draw[thick, color = blue, ->] (O) -- (mu1) node[anchor = north west]{$\mu_1$};
            \draw[thick, color = blue, ->] (O) -- (mu2) node[anchor = north east]{$\mu_2$};
            \fill[fill=blue,fill opacity=0.1] ($1.5*(mu1) + 1.5*(mu2)$) -- ($-1.5*(mu1) + 1.5*(mu2)$) -- ($-1.5*(mu1) + -1.5*(mu2)$) -- ($1.5*(mu1) + -1.5*(mu2)$) -- cycle ;
            \fill[] ($1.5*(mu1) + -1.5*(mu2)$) node[color = blue, right = 8pt] {${\scriptstyle \text{span}(\mu_1, \mu_2)}$};
            \draw[dashed, color = red, <->] (0, 0, -0.75) -- (0, 0, 0.75) node[anchor = north east]{${\scriptstyle A(x) - x \in \text{span}(\mu_1, \mu_2)^\perp}$};
        \end{tikzpicture}
        \caption{The Gaussian specialization (\cref{prop:normals}) of \cref{prop:orthogonal}. The indeterminacy set (in \textcolor{red}{red}) concentrates entirely on the z-axis, with means $\mu_1 = (1, 0, 0), \mu_2 = (0, 1, 0)$. The kernel of a plane in $\mathbb{R}^3$ is the perpendicular line through the origin.}
	\label{tikz_orthogonal_appx}
\end{figure}
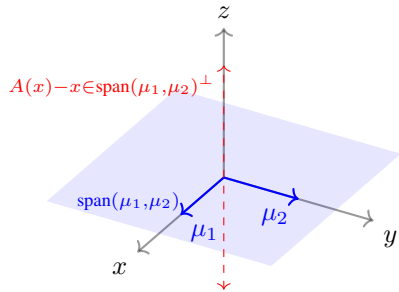

\begin{proposition}
    \label{prop:normals}
    In the multiple environments model described in \cref{thm:identifiability} (with $\latent = \mathbb{R}^d$), fix a base environment with distribution $\NormDist(0, \Sigma)$ and a subset $E^*$ of environments where $|E^*| = d' \leq d$, with distributions $\NormDist(\mu_e, \Sigma)$. Suppose $\{\mu_e\}_{e \in E^*}$ are linearly independent, and $(\mu_e)_i = 0$ for each $e$ and $i \notin d^*$ for some collection of dimensions $d^*$. Then, for any indeterminacy transformation $A_{a,b}$ it holds that
    \begin{align}
        (A_{a,b}(x))_{i\in d^*} = (x)_{i \in d^*} \quad a.e.
    \end{align}
\end{proposition}

\begin{proof}
    Similarly to the previous identifiability proofs, we appeal to \cref{cor:me_id_general} and analyze the shared automorphisms. For Gaussian distributions with a fixed covariance matrix varying by its mean, we have $T(x) = x$, and $\eta_i = \mu_c$. Using the base environment we have $\mu_0 = 0$. Now, arrange $\{\mu_i\}_{i=1}^{K}$ into the rows of a matrix $M$. By \cref{prop:orthogonal}, we have $(z - A(z)) \in \textrm{ker}M$. 

    $M$ has columns of $0$ corresponding to $d^*$ and linearly independent rows. Together, standard Gaussian elimination reveals that the reduced row echelon form of $M$ has the following form:
    \begin{align}
        \begin{bNiceMatrix}[first-row, last-col = 5]
            j \in d^* & j \notin d^*  & j \in d^*  & & \\
            1 & 0 & 0 & \cdots & \mu_1\\
            0 & 0 & 1 & \cdots & \mu_2\\
            \vdots & \vdots & \vdots & \cdots  &  \\
            0 & 0 & 0 & \cdots & \mu_d
        \end{bNiceMatrix}
    \end{align}
    The corresponding kernel is the space of vectors with $0$ entries for the indices $d^*$. Hence, $(z - A(z)) \in \textrm{ker}M$ implies $(A(z))_{i \in d^*} = (z)_{i \in d^*}$.

\end{proof}

\subsection{Identifiability of Equivariant Stochastic Mechanisms}
\label{proofs:esm}

In this section, we repeat the exercise of iVAE for the equivariant stochastic mechanisms model \citep{Ahuja2021}. In \citep{Ahuja2021}, weak identifiability of a temporal generative model is established. Adapted to our notation (note that the time indices represent environments), the model can be described:
\begin{align}
    \label{eqn:esm}
    X_t = f(Z_t)\;, \quad Z_{t+1} = m_t(Z_t, U_t)\;, Z_t \condind U_t\;, \quad t = 1, 2, \dots,
\end{align}
where $m_t \in \bfM: \latent \times [0,1] \to \latent$ are unknown mechanisms and $U_t$ are auxiliary noise variables. Note that this is in fact a noiseless generative model at the level of the observations, but where the underlying latent variable evolves according to a determinstic, but unknown mechanism $m_t$ and random noise $U_t$. To be clear, this means that $\genClass$ is fully flexible, while $\distClass_z^{t}$ is parametrized by an initial condition $P_1$, the distributions for $U_t$, and the mechanisms $m_t$. In what follows, we will assume a fixed $P_1$ and $U_t \sim U[0,1]$ as in \citep{Ahuja2021}, and leave $\distClass_z^{t}$ parametrized purely by the mechanisms $m_t$.

Denote the marginal distribution of $Z_{t}$ as $P_t$. In \citep{Ahuja2021}, identifiability of the generator $f$ is established up to pre-composition of some transformation $A$ such that $A \circ m_a(z,U) \equdist m_b(A(z), U)$ for $U\sim U[0,1]$, for all possible values of $z$, and $m_a, m_b \in \bfM$. Using our framework, we are able to show the following stronger identifiability result, using only observations from two time points $t = 1, 2$. 
\begin{proposition}
    \label{prop:esm}
    The model described by \eqref{eqn:esm} is identifiable up to $A_{a,b} \in \indet(\obsModel)$ satisfying
    \begin{gather}
        A_{a,b}(m_a(Z, U)) \equdist m_b(A_{a,b}(Z), U),
    \end{gather}
    for any $m_a, m_b \in \bfM$, $U \sim U[0,1]$ and any random variable $Z$ independent of $U$. 
\end{proposition}

\begin{proof}
We can analyze this model in our framework using just two time-points, $t = 1, 2$. We work on an augmented latent space $\tilde{\latent} = \latent \times [0,1]$ and treat the random variables $U_t$ as additional latent variables (i.e., as a ``noiseless'' case under our framework). For a generator $f: \latent \to \obs$, we extend $\tilde{f}: \tilde{\latent} \to \obs \times [0,1]$, $\tilde{f}(z, u) = (f(z), u)$. The identity extension ensures that $\tilde{f}$ is still injective, and is unique to $f$. Now suppose $f_a$ and $f_b$ are such that the distribution of $X_1$ and $X_2|X_1$ match. Note the marginal and conditional uniquely determine the joint, and hence we simply assume that the joint and hence marginal distributions of $X_1$ and $X_2$ match.

Let the joint distribution of $Z_1$ and $U$ be denoted $\pi_{Z_1, U}$. Since they are independent, we have that $\pi_{Z_1, U} = P_1 \otimes U[0,1]$.\footnote{This means that for a Borel product $B_z \times B_u$, where $B_z$, $B_u$ are Borel sets in their respective domains, we have $\pi_{X_1, U}(B_z \times B_u) = (P_1(B_z))(U[0,1](B_u))$.} We also extend the mechanism $m$ as $\tilde{m}(z, u) = (m(z, u), u)$, implying that $\tilde{m}^{-1}(B_z \times B_u) = m^{-1}(B_z) \times B_u$. Since $Z_2 = m(Z_1, U_1)$, this then implies that $P_2 = \pi_{Z_1, U} \circ \tilde{m}^{-1} = (P_1 \circ m^{-1})\otimes(U[0,1])$ (note the standard $m$ in the right-hand-side). The same applies to an extended indeterminacy transformation, i.e., $P_2 \circ \tilde{A}_{a,b}^{-1} = (P_1 \circ A_{a,b}^{-1})\otimes(U[0,1])$.

We now apply \cref{thm:indet:transport} to $t=1$, where $Z_1$ has fixed distribution $P_1$ (i.e., it is a singleton), and to $t=2$, where the latent distribution may vary with the mechanism $m_a$ or $m_b$, denoted $P_{2, a}, P_{2,b}$. As a result, we obtain 
\begin{gather}
    P_1 = A_{a,b \#}P_1, \quad P_{2,b} = \tilde{A}_{a,b \#}P_{2,a}.
\end{gather} 
Applying these identities simultaneously to $P_{2,b}$ gives
\begin{gather}
P_{2,b} = (m_{b \#}P_1 ) \otimes (U[0,1]) = (m_{b \#} A_{a, b \#}P_1) \otimes (U[0,1]) \\
P_{2,b} = \tilde{A}_{a, b \#}P_{2,a} = (A_{a,b \#} m_{a \#}P_1) \otimes (U[0,1]), 
\end{gather}
which by the properties of a product measure, means that
\begin{gather}
    m_{b \#} A_{a, b \#}P_1 = A_{a,b \#} m_{a \#}P_1.
\end{gather}
Writing the above in terms of their random variables, we have $m_b(A_{a,b}(Z), U) \equdist A_{a,b}(m_a(Z, U))$ for $Z$ with any fixed distribution $P_1$ independent of $U$. 
\end{proof}

Compared to the original proof, we are able to strengthen the result while weakening the assumptions due to our measure-theoretic framework as follows:

\begin{itemize}
    \item Letting $P_1$ be any point mass recovers the original identifiability result in \citep{Ahuja2021}. 
    \item Our proof structure, which can be found in the Appendix, follows the intuition originally laid out in \citep{Ahuja2021}, but we do not assume a diffeomorphic generator.
\end{itemize}

\section{\uppercase{Triangular transport maps}}
\label{sec:appx:triangular}

This section describes the triangular maps in \cref{sec:transports} in more detail, and proves the results therein. First, we define $f: \mathbb{R}^d \to \mathbb{R}^d$ to be a monotone increasing triangular map. This means that:

\begin{gather*}
    f(x) = \begin{bmatrix}
    f_1(x_1) \\
    f_2(x_1, x_2) \\
    \vdots \\
    f_d(x_1, \dots, x_d)
    \end{bmatrix},
\end{gather*}
where each $x_d \to f_d(x_{1:d-1}, x_d)$ is monotone increasing (hence invertible) for any $x_{1:d-1}$. The inverse of $f$ is as follows:
\begin{gather*}
    f^{-1}(x) = \begin{bmatrix}
    f_1^{-1}(x_1) \\
    f_2^{-1}(f_1^{-1}(x_1), x_2) \\
    \vdots \\
    f_d^{-1}(f_1^{-1}(x_1), f_2^{-1}(f_1^{-1}(x_1), x_2), \dots, x_d)
    \end{bmatrix}.
\end{gather*}
This is also a monotone increasing triangular map---the inverses of monotone increasing maps are also monotone increasing. Note the map described above is lower-triangular---upper-triangular maps are analogously defined. For the purposes of this section, a triangular map refers to a lower-triangular map. As long as all maps considered are either lower, or upper triangular, the same closure properties apply.

\subsection{Proofs of \cref{thm:krmap:id} and \cref{prop:krmap:ica:id}, and Kn\"othe--Rosenblatt transports}

It is well-known that if $\mu = \nu \circ f^{-1}$, where $\mu, \nu$ have strictly positive density and $f$ is a monotone increasing triangular map, then $f$ is equivalent to the Kn\"othe--Rosenblatt (KR) transport almost everywhere (see \citealt[Theorem 1]{Jaini2019}, for example). The KR transport is described recursively as follows. Let $F_{\mu}(x_m|x_{1:m-1})$ be the conditional CDF of the $m-$th component of $\mu$ on the preceding components. Because $\mu$ has strictly positive density, $F_{\mu}$ is monotone increasing. Then, the $m$-th component of the KR transport is as follows:
\begin{gather*}
    K_m(x_{1:m-1}, x_m) = F_{\nu}^{-1}\{ F_{\mu}(x_m|x_{1:m-1}) \mid K_1(x_1), \dots, K_{m-1}(x_{1:m-1})\}.
\end{gather*}

That is, $K_m$ sends $x_m$ through the conditional CDF of $\mu$ on $x_{1:m-1}$, and back through the inverse conditional CDF of $\nu$ on $y_{1:m-1} = (K_1(x_1), \dots, K_{m-1}(x_{1:m-1}))$. This CDF transform is the unique (almost everywhere) monotone increasing transport map between the 1-dimensional unique (almost everywhere) regular conditional probabilities.

It is clear that the map $K$ defined by its components $K_m$ is monotone increasing triangular. Since it is the unique such map transporting $\mu$ to $\nu$, and triangular monotone increasing maps are closed under inverses and compositions, it must be that:

\begin{itemize}
    \item For $K$ the KR transport from $\mu$ to $\nu$, $K^{-1}$ is the KR transport from $\nu$ to $\mu$. 
    \item For measures $\mu, \nu, \pi$, if $K_1$ is the KR transport from $\mu$ to $\nu$, $K_2$ is the KR transport from $\nu$ to $\pi$, then $K_1 \circ K_2$ is the KR transport from $\mu$ to $\pi$.
\end{itemize}

Now, it is clear that if $f_a$, $f_b$ are KR transports, their generator transformation $\genTrans_{a,b}$, and hence all indeterminacy transformations, are also KR transports. This drives our results from the main paper---the proofs of our results are trivial given the observations above.

\begin{proof}[Proof of \cref{thm:krmap:id}]
    $\indet(\genClass)$ is the set of functions equal almost everywhere to KR maps transporting between measures in $\distClass_z$. Since $\distClass_z = \{P_z\}$, $\indet(\genClass)$ is the set of functions equal almost everywhere to the KR transport from $P_z$ to itself---the identity map. By Theorem \ref{thm:id_general}, $\indet(\obsModel) = \widetilde{\id}_z$. 
\end{proof}

\begin{proof}[Proof of \cref{prop:krmap:ica:id}]
    $\indet(\obs)$ is the set of functions equal almost everywhere to KR transports between measures in $\distClass_z$. Let $P_{z,a}$ and $P_{z,b}$ be two such measures, which by assumption have independent components. By its construction, it is clear that $K_m$ depends only on $x_m$ in any KR transport $K$ between $P_{z,a}$ and $P_{z,b}$. Such a map is monotone increasing and diagonal---hence by \cref{thm:id_general}, $\indet(\obsModel)$ consists of invertible, component-wise transformation. 
\end{proof}

\section{\uppercase{Linear examples}} \label{sec:appx:examples}

We conclude the appendix by analyzing the identifiability of the linear examples of \cref{sec:fa:ica} in detail. The factor analysis example provides intuition on why multiple environments reduces the indeterminacy set, and the ICA example provides intuition on how the two intersecting sets in \cref{thm:id_general} can be manipulated in synergy.


\subsection{Example: Linear-Gaussian}

We present a simple example of using multiple environments, and a basis of priors, to obtain identifiability, using only linear algebra concepts. This example also provides intuition for the minimality of the number of environments. That is, for this 2-d latent space, three environments is enough to obtain strong identifiability, while two environments is insufficient. 

Suppose two competing linear generative models for a random vector $x \in \mathbb{R}^{10}$ with latent vector $z \in \mathbb{R}^2$, for data arising from three environments indexed by $e = 1, 2, 3$:
\begin{multicols}{2}
\noindent
\begin{align}
\nonumber
z^{(e)} \sim N(\mu, I_{2\times 2})\\  \nonumber
\epsilon \sim N(\mu, I_{10 \times 10}) \\  
y^{(e)} = \alpha + Fz^{(e)} + \epsilon 
\end{align} 
\begin{align}
\nonumber
z^{(e)} \sim N(\mu_e, I_{2\times 2}) \\ \nonumber
\epsilon \sim N(\mu, I_{10 \times 10}) \\
x^{(e)} = \alpha + Fz^{(e)} +\epsilon.
\end{align}
\end{multicols}
The left model is a single environment model, while in the right model, two of the $\mu_e$ are linearly independent, i.e., a multiple environment model. Note that the generator function here is 
\begin{gather}
g(z) = \alpha + Fz,
\end{gather}
where $F$ is a full rank $10 \times 2$ matrix, and $\alpha$ is an offset vector in data space, fixed for all environments. For each environment we have the marginal distribution under the multiple environment model:
\begin{gather}
x^{(e)} \sim N(\alpha + F\mu_e, FF^\top + I_{10\times 10})
\end{gather}
Recall the Gaussian distribution is characterized entirely by its mean and covariance---that is, for marginal distributions parametrized by $\theta_1 = (\alpha_1, F_1), \theta_2 = (\alpha_2, F_2)$:
\begin{gather}
P_{\theta_1, e} = P_{\theta_2, e} \iff \alpha_1 + F_1\mu_e = \alpha_2 + F_2\mu_e, \quad F_1F_1^\top = F_2F_2^\top .
\end{gather}
To say that this model is strongly identifiable means that the right-hand-side equalities for each $e$ imply $\alpha_1 = \alpha_2$ and $F_1 = F_2$. 

In the single environments model, there are the following constraints:
\begin{gather}
    \alpha_1 + F_1\mu = \alpha_2 + F_2\mu \\
    F_1F_1^\top = F_2F_2^\top 
\end{gather}
The single environments model is not identifiable. For example, let $R$ be an orthogonal (rotation) matrix, then, let $F_2 = F_1R$ and $\alpha_2 = \alpha_1 - F_1R\mu + F_1\mu$. We have
\begin{gather}
    \alpha_2 + F_2\mu = \alpha_1 - F_1R\mu + F_1\mu + F_1R\mu  = \alpha_1 + F_1\mu \\
    F_2F_2^\top = F_1RR^\top F_2^\top = F_1F_1^\top,
\end{gather}
where the last equality is due to $R$ being an orthogonal matrix. This is a classical case of exploiting the rotational invariance of the Gaussian to construct a non-identifiable example. 

Now, we analyze the multiple environment model. To be explicit, the three environments impose the following constraints in the multiple environments model:
\begin{gather}
    \alpha_1 + F_1\mu_1 = \alpha_2 + F_2\mu_1 \\
    \alpha_1 + F_1\mu_2 = \alpha_2 + F_2\mu_2 \\
    \alpha_1 + F_1\mu_3 = \alpha_2 + F_2\mu_3 \\
    F_1F_1^\top = F_2F_2^\top,
\end{gather}
We can show directly that these constraints imply that $\alpha_1 + F_1z = \alpha_2 + F_2z$. First, assume that $\mu_1$ and $\mu_2$ are the linearly independent pair. Then, taking differences,
\begin{gather}
    F_1(\mu_1 - \mu_3) = F_2(\mu_1 - \mu_3) \\
    F_1(\mu_2 - \mu_3) = F_2(\mu_2 - \mu_3) \\
    F_1F_1^\top = F_2F_2^\top . 
\end{gather}
Written in matrix form, the first two constraints read 
\begin{gather}
    F_1M = F_2M \implies F_1 = F_2,
\end{gather}
since $\mu_1 - \mu_3$ and $\mu_2 - \mu_3$ remain linearly independent, and hence $M$ is invertible. It immediately follows from the original constraints that $\alpha_1 = \alpha_2$ also. 

The above analysis showed that, for identifiability, a single environment was insufficient, while three environments was adequate. This begs the question, what about two environments? In other words, is the three environment constraint minimal? 

Consider a model with two environments with means $\mu_1$, $\mu_2$. By the arguments above, it imposes the following constraints:
\begin{gather}
    \alpha_1 + F_1\mu_1 = \alpha_2 + F_2\mu_1 \\
    \alpha_1 + F_1\mu_2 = \alpha_2 + F_2\mu_2 \\
    F_1F_1^\top = F_2F_2^\top.
\end{gather}

Can we construct an non-identifiable example? Let $F_2 = F_1R$, $\alpha_2 = \alpha_1 - F_1R\mu_1 + F_1\mu_1$ as in the single-environment case. Clearly, these satisfy the first and third constraint for any orthogonal matrix $R$. We aim to find a specific rotation matrix that also satisfies the second constraint. Observe that:
\begin{gather}
    \alpha_2 + F_2\mu_2 = \alpha_1 - F_1R\mu_1 + F_1\mu_1 + F_1R\mu_2 \\
    = \alpha_1 + F_1\mu_1 + F_1R(\mu_2 - \mu_1).
\end{gather}
Let $x$ be a vector orthogonal to $\mu_2 - \mu_1$, standardized such that $\|x\|^2 = \|\mu_2 - \mu_1\|^2$. Consider
\begin{gather}
    R = \frac{1}{\|\mu_2 - \mu_1\|^2}\begin{bmatrix}
        \vert & \vert \\
        \mu_2 - \mu_1 & x \\
        \vert & \vert 
    \end{bmatrix}
    \begin{bmatrix}
        1 & 0 \\
        0 & -1
    \end{bmatrix}
    \begin{bmatrix}
        \vert & \vert \\
        \mu_2 - \mu_1 & x \\
        \vert & \vert 
    \end{bmatrix}^\top .
\end{gather}
This is the eigendecomposition of an orthogonal matrix (it is the product of orthogonal matrices) with eigenvalues $1$ and $-1$, and corresponding eigenvectors $\mu_2 - \mu_1$ and $z$. Since it is an eigenvector, we have $R(\mu_2 - \mu_1) = \mu_2 - \mu_1$.\footnote{R is essentially a rotation matrix with axis $(\mu_2 - \mu_1)$.} Then, we have
\begin{gather}
    \alpha_2 + F_2\mu_2 = \alpha_1 + F_1\mu_1 + F_1\mu_2 - F_1\mu_1 = \alpha_1 + F_1\mu_2,
\end{gather}
which satisfies the second constraint as desired. This shows that three environments are required, and hence minimal for strong identifiability of this model.

Note that such a construction will not work for the three-environment model. For three environments, the rotation has to satisfy both
\begin{gather}
R(\mu_3 - \mu_1) = \mu_3 - \mu_1 \\
R(\mu_2 - \mu_1) = \mu_2 - \mu_1,
\end{gather}
that is, the eigenspace of $R$ associated to the eigenvalue $1$ spans $\mathbb{R}^2$, i.e., it is the identity.


\subsection{Linear, non-Gaussian ICA}
\label{sec::linearica}

Consider a generative model (Equation \eqref{eq:model}) with $\latent = \bbR^{d_z}$ and $\obs = \bbR^{d_x}$. Assume $d_x \geq d_z$. Let the generator parameter space be
$\mathcal{F} = \{ A \in \mathbb{R}^{d_x \times d_z}; rank(A) = d_z \}$. That is, the generators are full-rank linear transformations, and hence injective. Let the prior parameter space be 
\begin{align}
\mathcal{P}_z = \{p(z) = \prod_{i=1}^{d_z} p_i(z) ; p_i \text{ are non-Gaussian, and not a point mass} \},
\end{align} i.e., probability distributions on $\mathbb{R}^{d_z}$ with a density, and the density factorizes as independent, non-Gaussian components. 

The identifiability of this problem was first studied in \citep{Comon1994}.\footnote{In the original analysis, the model is fit according to a criteria maximizing the independence between components, and also one component of the prior is allowed to be Gaussian. For simplicity, we will simply study the implications of matching observational marginal distributions (i.e., maximum likelihood) and where all prior components are non-Gaussian.} In their analysis, identifiability is established up to pre-multiplication of a diagonal matrix and a permutation. That is, for generators $F_a$, $F_b \in \mathcal{F}$ with $P_{z,a}, P_{z,b} \in \mathcal{P}_z$, if the marginal distributions on $\obs$ match, then $F_a = F_b \Lambda P$, where $\Lambda$ is an invertible diagonal matrix and $P$ is a permutation matrix. 

Under our framework, i.e., Lemma \ref{thm:indet:transport}, we must have that $P_{z,b} = P_{z,a} \circ \genTrans_{a,b}^{-1}$, where $F_a = F_b \circ \genTrans_{a,b}$. Using our framework, we now show that $\genTrans_{a,b} = \Lambda P$ as above. The identifiability result obtained in \citep{Comon1994} rests on the following result (restated and re-proved to match our notation):

\begin{theorem}[Theorem 10, \citep{Comon1994}]
    \label{thm:comon}
    Let $z$ be a random vector with factorized density. Let $x = Cz$, such that $x$ also has factorized density. Then, $z_j$ is non-Gaussian if the $j$-th column has at most one non-zero entry.
\end{theorem}

\begin{proof}
    We require Theorem 19 from \citep{Comon1994}.
    \begin{lemma}[\citep{Comon1994}, Darmois' Theorem]
        \label{lem:darmois}
        Define two random variables $Z_1$ and $Z_2$ as
        \begin{gather}
            Z_1 = \sum_{i} a_i z_i \quad \quad Z_2 = \sum_{i} b_i z_i,
        \end{gather}
        where $z_i$ are independent random variables, i.e., their joint distribution factorizes. Then, if $Z_1$ and $Z_2$ are independent, all variables $z_j$ for which $a_j b_j \neq 0$ are Gaussian.
    \end{lemma}
    Now, let $z$ be a random vector with factorized density and $x = Cz$, where $x$ has factorized density also. Note that this implies any $x_i$, $x_k$ are independent for $i\neq k$. We have that
    \begin{gather}
        x_i = \sum_{j} C_{ij} z_j \quad \quad x_k = \sum_{j} C_{kj} z_j,
    \end{gather}
    and hence by Lemma \ref{lem:darmois}, if $z_j$ is non-Gaussian, it must be that $C_{ij}C_{kj} = 0$. This holds for each $i \neq k$, and hence, the $j$-th column has at most one non-zero entry. 
\end{proof}

Recall the definition of $\mathcal{P}_z$ is such that any prior must factorize and be non-Gaussian. Then, Theorem \ref{thm:comon} implies that 
\begin{gather}
    \mathcal{A}(\mathcal{P}_z) \cap \mathbb{R}^{d_z \times d_z} = \{A \in \mathbb{R}^{d_z \times d_z} \mid A \text{ has no column with more than one nonzero element} \}. 
\end{gather}
That is, any linear isomorphisms between two priors must have no column with more than one nonzero element. Now, note that for any $\genTrans_{a,b} \in \mathcal{A}({\mathcal{F}})$, we have
\begin{gather}
    \genTrans_{a,b} = f_b^{-1} \circ F_a,
\end{gather}
where $f_b^{-1}$ is the restriction of the linear map represented by the pseudoinverse $F_b^\dagger$ to the range of $\mathcal{F}$. By Lemma \ref{lem:indeterminacyismsbl}, $\genTrans_{a,b}$ is an invertible linear map and hence full rank. Finally, we conclude that for any $\genTrans_{a,b} \in \mathcal{A}(\mathcal{F}) \cap \mathcal{A}(\mathcal{P})$, $\genTrans_{a,b}$ must have exactly one nonzero element in each column. We can then apply a permutation $P$ such that $P\genTrans_{a,b} = \Lambda$, where $\Lambda$ is diagonal. Finally, we obtain $\genTrans_{a,b} = P^\top \Lambda$, where $P^\top$ is a permutation matrix and $\Lambda$ is diagonal and invertible.

\vfill

\end{document}